\documentclass{article}

\usepackage{arxiv} 

\usepackage[utf8]{inputenc} 
\usepackage[T1]{fontenc}    
\usepackage{url}            
\usepackage{booktabs}       
\usepackage{amsfonts}       
\usepackage{nicefrac}       
\usepackage{microtype}      
\usepackage{graphicx}
\usepackage{subcaption}
\usepackage{hyperref}
 \usepackage{natbib} 

\usepackage{amsmath}
\usepackage{amssymb}
\usepackage{mathtools}
\usepackage{amsthm}
\usepackage{bm}
\usepackage{mathrsfs}
\usepackage{dsfont}
\usepackage{threeparttable}
\usepackage{siunitx}
\usepackage{pgfplots}
\usepackage{colortbl}
\definecolor{Gray}{gray}{0.9}
\pgfplotsset{compat=1.18}
\usepackage[capitalize,noabbrev]{cleveref}

\usepackage{algorithm}
\usepackage{algorithmic}

\theoremstyle{plain}
\newtheorem{theorem}{Theorem}[section]
\newtheorem{proposition}[theorem]{Proposition}
\newtheorem{lemma}[theorem]{Lemma}

\theoremstyle{definition}
\newtheorem{definition}[theorem]{Definition}
\newtheorem{assumption}[theorem]{Assumption}
\newtheorem{axiom}[theorem]{Axiom}
\theoremstyle{remark}
\newtheorem{remark}[theorem]{Remark}

\DeclareMathOperator*{\argmin}{arg\,min}

\DeclareMathOperator{\tr}{tr}
\DeclareMathOperator{\expm}{expm}
\DeclareMathOperator{\logdet}{log det}
\newcommand{\norm}[1]{\left\lVert#1\right\rVert}

\newcommand{\abs}[1]{\left\lvert#1\right\rvert}
\newcommand{\R}{\mathbb{R}}
\newcommand{\matI}{\mathbf{I}}
\newcommand{\mA}{\mathbf{A}}
\newcommand{\mW}{\mathbf{W}}
\newcommand{\mU}{\mathbf{U}}
\newcommand{\mH}{\mathbf{H}}

\newcommand{\mO}{\mathcal{O}}

\newcommand{\frobinner}[2]{\left\langle #1, #2 \right\rangle_F}
\newcommand{\hadamard}{\circ}

\newcommand{\smoothabs}[2]{\sqrt{#1^2 + #2^2}}

\title{Smooth, Sparse, and Stable: Finite-Time Exact Skeleton Recovery via Smoothed Proximal Gradients}

\usepackage{authblk}

\author[1]{Rui Wu\thanks{Email: \texttt{wurui22@mail.ustc.edu.cn}}}
\author[1]{Yongjun Li\thanks{Corresponding author. Email: \texttt{lionli@ustc.edu.cn}}}
\affil[1]{School of Management, University of Science and Technology of China, Hefei, Anhui, China}

\date{}

\begin{document}

\maketitle

\begin{abstract}
    Continuous optimization has significantly advanced causal discovery, yet existing methods (e.g., NOTEARS) generally guarantee only asymptotic convergence to a stationary point. This often yields dense weighted matrices that require arbitrary post-hoc thresholding to recover a DAG. This gap between continuous optimization and discrete graph structures remains a fundamental challenge.
    
    In this paper, we bridge this gap by proposing the \textbf{Hybrid-Order Acyclicity Constraint (AHOC)} and optimizing it via the \textbf{Smoothed Proximal Gradient (SPG-AHOC)}.
    Leveraging the \textbf{Manifold Identification Property} of proximal algorithms, we provide a rigorous theoretical guarantee: the \textbf{Finite-Time Oracle Property}.
    We prove that under standard identifiability assumptions, SPG-AHOC recovers the \textbf{exact DAG support} (structure) in finite iterations, even when optimizing a smoothed approximation.
    This result eliminates structural ambiguity, as our algorithm returns graphs with \textbf{exact zero entries} without heuristic truncation.
    Empirically, SPG-AHOC achieves state-of-the-art accuracy and strongly corroborates the finite-time identification theory.
\end{abstract}


\section{Introduction}
\label{sec:intro}

The transition from correlation to causation relies on dissecting the precise structure of interactions. However, a chasm exists between the discrete nature of graphs and the continuous optimization tools used to find them.
Historically, this problem was addressed using combinatorial search algorithms \citep{Chickering2002}, which scale poorly with the number of variables $d$.
The field witnessed a paradigm shift with the introduction of NOTEARS \citep{Zheng2018NOTEARS}, which reformulated the discrete Directed Acyclic Graph (DAG) constraint into a continuous equality constraint, enabling the use of efficient continuous optimization techniques.

\textbf{The Gap: Asymptotic Convergence vs. Exact Structure.}
Despite the computational advances, a critical disconnect remains. Existing methods typically guarantee only \textit{asymptotic convergence} to a stationary point. In practice, this results in dense weighted adjacency matrices where "zero" edges are represented by small non-zero values (e.g., $10^{-4}$).
This necessitates arbitrary \textbf{post-hoc thresholding} to recover a valid DAG.
In scientific discovery, this introduces \textbf{Scientific Ambiguity}: does a small weight represent a weak causal link or numerical noise?
Furthermore, standard constraints suffer from \textbf{Structural Constraint Instability (SCI)}, facing a trade-off between vanishing gradients (hindering sparsity) and exploding gradients (causing numerical divergence).

\textbf{Our Solution: Topological Locking.}
To bridge this gap, we propose a theoretically grounded framework that ensures \textbf{Finite-Time Exact Recovery}.
We introduce the \textbf{Hybrid-Order Acyclicity Constraint (AHOC)}, which satisfies a complete set of stability axioms.
Crucially, we optimize this constraint using the \textbf{Smoothed Proximal Gradient (SPG-AHOC)}.
We identify a mechanism we term \textbf{Topological Locking} (visualized in Figure \ref{fig:manifold_id}): while standard gradient descent asymptotically approaches zero, the proximal operator explicitly "snaps" the optimization trajectory onto the sparse manifold in finite steps once it enters a \textbf{Geometric Stability Zone}.
This allows our method to return graphs with \textbf{exact zero entries}, eliminating the need for heuristic truncation.

\textbf{Contributions.} Our main contributions are threefold:
\begin{itemize}
    \item \textbf{Axiomatic Framework \& Impossibility Theorem:} We identify the root causes of instability in existing methods and prove that standard smooth constraints cannot simultaneously satisfy gradient non-vanishing and boundedness axioms. We propose the AHOC family to resolve this.
    \item \textbf{Finite-Time Oracle Property:} We provide a rigorous proof that SPG-AHOC achieves exact support recovery in finite iterations. We show that the "Topological Locking" effect of the proximal operator makes the support recovery robust to the smoothing bias required for optimization.
    \item \textbf{SOTA Performance \& Robustness:} Empirically, SPG-AHOC achieves state-of-the-art accuracy on synthetic benchmarks and real-world data (Sachs), while demonstrating superior computational efficiency in low-to-mid dimensions compared to unconstrained baselines.
\end{itemize}

\section{Theoretical Analysis of Acyclicity Constraints}
\label{sec:theory}

In this section, we analyze the numerical properties of differentiable acyclicity constraints and propose our solution.

\subsection{Our Method: An Axiomatic Diagnosis}
\label{subsec:axioms}

We propose three fundamental axioms for stable acyclicity constraints $h(\mW)$:

\begin{axiom}[Correctness]
    \label{axiom:correctness}
    $h(\mW) = 0 \iff \mathcal{G}(\mW) \text{ is a DAG}$.
\end{axiom}
\textit{Consequence of Violation: Converges to incorrect structures.}

\begin{axiom}[$L_1$-Synergy]
    \label{axiom:l1_synergy}
    $\liminf_{\delta \to 0^+} \norm{\nabla_\mW h(\delta \cdot \mU)}_F > 0$.
\end{axiom}
\textit{Consequence of Violation (Type I SCI): $L_1$ gradient dominates, hindering sparse DAG discovery.}

\begin{axiom}[Numerical Boundedness]
    \label{axiom:boundedness}
    $\sup_{\mW} |h(\mW)| < \infty$ and $\sup_{\mW} \norm{\nabla_\mW h(\mW)}_F < \infty$.
\end{axiom}
\textit{Consequence of Violation (Type II SCI): Exploding gradients cause numerical divergence (`NaN`/`Inf`). This can be fatal unless the optimizer employs a "safety net".}

\subsection{Core Finding 1: The \texorpdfstring{$\mathcal{H}$}{H} Family and its Incompatibility}
\label{subsec:h_family_and_incompatibility}

We reveal a fundamental \textbf{Impossibility Theorem} for the standard family of constraints.

\begin{definition}[Standard Hadamard Constraint Class $\mathcal{H}$]
    \label{def:h_family_rigorous}
    Let $\mathcal{H}$ be the family of spectral constraints where the dependency on $\mW$ acts exclusively through the entry-wise magnitude matrix $\mA = \mW \hadamard \mW$ (or $\mA = \abs{\mW}$). Specifically, $h(\mW) = F(\mA)$ where $F$ is a $C^1$ smooth function on the cone of non-negative matrices.
\end{definition}

\begin{theorem}[Refined Incompatibility Theorem]
    \label{thm:incompatibility}
    For any constraint $h \in \mathcal{H}$ satisfying $h(\mathbf{0})=0$, if the core function $F$ is smooth at the origin, then $h(\mW)$ cannot simultaneously satisfy Axiom \ref{axiom:l1_synergy} ($L_1$-Synergy) and Axiom \ref{axiom:boundedness} (Numerical Boundedness).
\end{theorem}
\textit{Implication: Standard constraints face an unavoidable trade-off. Quadratic dependencies (required for smoothness at origin) inherently conflict with the linear sensitivity required by $L_1$ regularization.}

\begin{figure}[t]
    \centering
    \includegraphics[width=0.7\columnwidth]{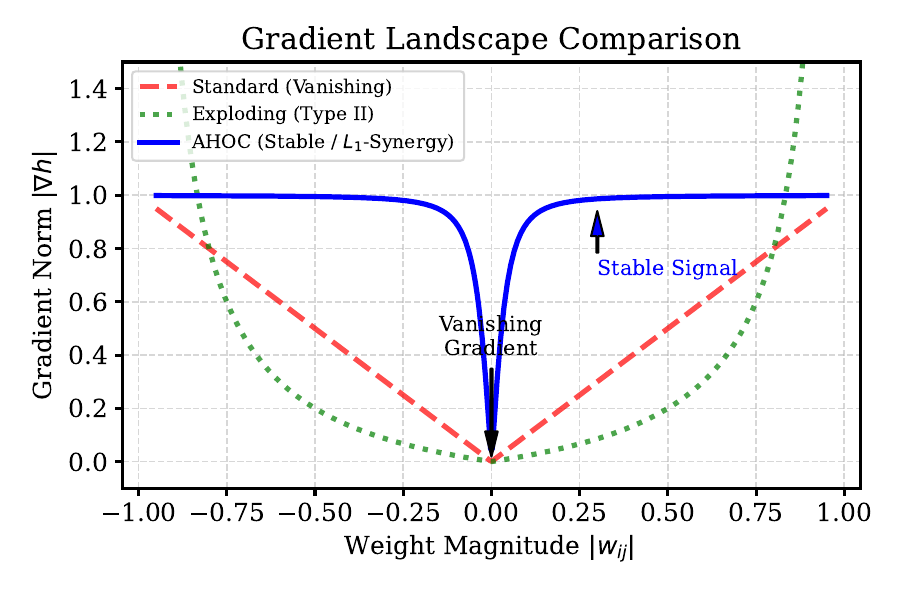}
    \vspace{-0.1in}
    \caption{\textbf{Visualizing Structural Constraint Instability (SCI).}
    \textit{Left (Red/Green):} Standard constraints suffer from either vanishing gradients at the origin (Type I) or exploding gradients near the boundary (Type II).
    \textit{Right (Blue):} Our AHOC formulation maintains a stable, non-vanishing gradient signal near zero due to the smoothed $L_1$-Synergy (Axiom 2), while preventing explosion via normalization (Axiom 3).}
    \label{fig:sci_visualization}
\end{figure}

\subsection{An Intermediate Step: AAC Solves Type II SCI}
\label{subsec:aac}

Introducing Adaptive Scale Normalization (ASN), $\phi_{\text{ASN}}(\mW; \epsilon) = \frac{\mA}{\norm{\mA}_F + \epsilon}$, yields the AAC constraint:
$$ h_{\text{AAC}}(\mW; \epsilon) \coloneqq \tr\left( \expm\left( \phi_{\text{ASN}}(\mW; \epsilon) \right) \right) - d $$

\begin{theorem}[Axiomatic Analysis of AAC]
    \label{thm:aac_analysis}
    AAC satisfies Axioms \ref{axiom:correctness} and \ref{axiom:boundedness}, but fails Axiom \ref{axiom:l1_synergy}.
\end{theorem}
\textit{Implication: AAC fixes Type II SCI (satisfies A3) but not Type I SCI (fails A2), as its gradient is $\mO(\norm{\mW}_F / \epsilon)$ and still vanishes.}

\subsection{The Theoretical Solution: Hybrid-Order Constraints (AHOC/S-AHOC)}
\label{subsec:ahoc_construction}

To satisfy Axiom 2 ($L_1$-Synergy), which requires a non-vanishing gradient signal as $\mW \to 0$, we introduce a first-order term via the Hybrid-Order Core (HOC):
$$ \mathbf{M}(\mW; \alpha) \coloneqq \alpha (\mW \hadamard \mW) + (1-\alpha) \abs{\mW}, \quad \alpha \in [0, 1) $$
Applying normalization yields two variants:

\begin{definition}[AHOC and S-AHOC Constraints]
    \label{def:ahoc_family}
    Let $h_{\exp}(\mathbf{M}) = \tr(\expm(\mathbf{M})) - d$.
    \begin{itemize}
    \item \textbf{AHOC (ASN Core):}
    \[
        h_{\text{AHOC}}(\mW; \alpha, \epsilon) \coloneqq h_{\exp}\left(\frac{\mathbf{M}(\mW; \alpha)}{\norm{\mathbf{M}(\mW; \alpha)}_F + \epsilon}\right)
    \]
    \item \textbf{S-AHOC (S-ASN Core):}
    \[
        h_{\text{S-AHOC}}(\mW; \alpha) \coloneqq \tr\left( \expm\left( \frac{\mathbf{M}(\mW; \alpha)}{\norm{\mathbf{M}(\mW; \alpha)}_F + 1} \right) \right) - d
    \]
\end{itemize}
\end{definition}

\begin{theorem}[Axiomatic Analysis of AHOC/S-AHOC]
    \label{thm:ahoc_analysis}
    Both $h_{\text{AHOC}}$ and $h_{\text{S-AHOC}}$ satisfy all three axioms: Correctness (A1), $L_1$-Synergy (A2), and Numerical Boundedness (A3).
\end{theorem}
\textit{Implication: The AHOC/S-AHOC family is the first theoretically complete solution, resolving both SCI types.}

\subsection{The Smoothness Barrier: Theory vs. Optimization}
\label{subsec:ahoc_failure}

While the S-AHOC constraint theoretically satisfies all three stability axioms, it represents a \textit{Platonic ideal} rather than a directly optimizable function.
Specifically, because the core $\mathbf{M}(\mW)$ involves $\abs{\mW}$, the constraint $h_{\text{S-AHOC}}$ is non-differentiable at any point where $\mW_{ij}=0$.
This violates the standard assumption of Proximal Gradient methods, which require the smooth part of the objective to have a Lipschitz-continuous gradient \citep{Beck2017}.
Directly applying subgradient methods would forfeit the convergence rates we desire.
Therefore, SPG-AHOC (Section \ref{subsec:spg_ahoc}) serves as the necessary algorithmic bridge: it optimizes a smoothed approximation to recover the properties of the ideal S-AHOC via the topological locking mechanism.

\begin{remark}[Theoretical Barrier for Proximal-LS]
    \label{rem:prox_ls_failure}
    The convergence guarantees of Proximal-LS strictly require that the smooth component of the objective function, let's call it $f(\mW)$ (which includes $h_{\text{AHOC}}$), has a Lipschitz continuous gradient \citep{Beck2017}. However, due to the $\abs{\mW}$ term, $h_{\text{AHOC}}$ is non-differentiable at any $\mW_{ij}=0$, and thus its gradient is not well-defined, let alone Lipschitz continuous, over its entire domain.
\end{remark}

\subsection{Our Algorithmic Solution: Smoothed Proximal Gradient for AHOC (SPG-AHOC)}
\label{subsec:spg_ahoc}

To overcome the optimization barrier, we propose the Smoothed Proximal Gradient for AHOC (SPG-AHOC) algorithm. The core idea is to apply Proximal-LS not to the original non-smooth problem, but to a smooth approximation of it. We replace the non-smooth absolute value function $\abs{x}$ with a smooth counterpart, $\text{smooth\_abs}(x; \delta) = \sqrt{x^2 + \delta^2}$, where $\delta > 0$ is a small smoothing parameter.

This leads to a smoothed HOC core:
$$ \tilde{\mathbf{M}}(\mW; \alpha, \delta) \coloneqq \alpha (\mW \hadamard \mW) + (1-\alpha) \smoothabs{\mW}{\delta} $$
And the corresponding smooth approximation for the AHOC constraint:
$$ \tilde{h}_{\text{AHOC}}(\mW; \alpha, \epsilon, \delta) \coloneqq h_{\exp}\left(\frac{\tilde{\mathbf{M}}(\mW; \alpha, \delta)}{\norm{\tilde{\mathbf{M}}(\mW; \alpha, \delta)}_F + \epsilon}\right) $$

\paragraph{The Smooth-Non-Smooth Decoupling.}
Methodologically, we propose the \textbf{Decoupled Optimization Paradigm}. Unlike prior works that enforce smoothness globally (sacrificing exact sparsity) or ignore smoothness (risking instability), we mathematically decouple the \textbf{geometric feasibility} (via smoothed AHOC) from the \textbf{structural identification} (via non-smooth $L_1$).
Specifically, we formulate the problem as:
$$ \min_{\mW} \underbrace{\mathcal{L}_{\text{fit}}(\mW) + \text{ALM}_{\delta}(\mW)}_{F(\mW) \text{ (Smooth Component)}} + \underbrace{\lambda_1 \norm{\mW}_1}_{R(\mW) \text{ (Non-Smooth Component)}} $$
Here, the smoothing via $\delta$ ensures that the augmented Lagrangian part $F(\mW)$ has a Lipschitz continuous gradient, while the regularization part $R(\mW)$ remains non-differentiable. We show that this decoupling is necessary to bypass the Incompatibility Theorem and achieve finite-time exact recovery.

See Algorithm \ref{alg:spg_ahoc} for the implementation.

\begin{algorithm}[tb]
    \caption{Smoothed Proximal Gradient (SPG) Step}
    \label{alg:spg_ahoc}
    \begin{algorithmic}
        \STATE {\bfseries Input:} Iterate $\mW_k$, params $\mu_t, \rho_t, \lambda_1$, smoothing $\delta$.
        \STATE {\bfseries Define:} Smoothed ALM part $\tilde{L}(\mW) = \mathcal{L}_{\text{fit}}(\mW) + \mu_t \tilde{h}_{\text{AHOC}}(\mW) + \frac{\rho_t}{2} \tilde{h}_{\text{AHOC}}(\mW)^2$.
        \STATE Compute gradient $G_k = \nabla_\mW \tilde{L}(\mW_k)$.
        \STATE Initialize step size $\eta$.
        \WHILE{Linesearch Condition Not Met}
            \STATE $\mW_{cand} \leftarrow \text{prox}_{\lambda_1 \eta}(\mW_k - \eta G_k)$
            \IF{Sufficient Decrease in $\tilde{L}(\mW_{cand}) + \lambda_1 \norm{\mW_{cand}}_1$}
                \STATE \textbf{break}
            \ELSE
                \STATE Decrease $\eta$
            \ENDIF
        \ENDWHILE
        \STATE {\bfseries Output:} $\mW_{k+1} = \mW_{cand}$.
    \end{algorithmic}
\end{algorithm}
\begin{remark}[The Logic of Exactness: Why Smoothing Doesn't Break Sparsity]
    \label{rem:exactness_logic}
    It is crucial to distinguish which components are smoothed. We smooth the \textit{constraint} $h(\mW)$ to ensure gradients are well-defined. However, we do \textbf{not} smooth the $L_1$ penalty.
    The finite-time identification property comes solely from the non-smooth $L_1$ norm handled by the proximal operator.
    Even if the smoothed constraint $\tilde{h}(\mW)$ introduces a microscopic bias $\mathcal{O}(\delta)$, the proximal operator's thresholding band $[-\lambda \eta, \lambda \eta]$ absorbs this noise, snapping small non-zero entries (artifacts of smoothing) to exact zeros.
\end{remark}

\begin{remark}[Mechanism: Topological Locking]
    \label{rem:manifold_id}
    We identify a phenomenon we term \textbf{Topological Locking}: for $\delta$ below a critical threshold derived from the Strict Dual Feasibility gap, the proximal operator 'locks' onto the correct sparse manifold despite the smooth approximation. This implies that exact discrete structure recovery is possible via continuous approximations, provided we operate within the \textbf{Geometric Stability Zone} of the regularizer.
\end{remark}

\subsection{Convergence Analysis and The Smoothing-Optimization Trade-off}
\label{subsec:spg_convergence}

The SPG-AHOC algorithm optimizes the smoothed ALM objective $L_\delta(\mW)$. We provide a rigorous analysis of its convergence behavior and the fundamental trade-off introduced by the smoothing parameter $\delta$.

\begin{theorem}[Convergence of SPG-AHOC Inner Loop]
    \label{thm:spg_convergence}
    Let the data fitting term $\mathcal{L}_{\text{fit}}$ have a Lipschitz continuous gradient. For any \textbf{fixed} smoothing parameter $\delta > 0$, the gradient of the smoothed ALM objective $\nabla \tilde{L}(\mW; \mu, \rho)$ is Lipschitz continuous with constant $L_\delta$.
    The Backtracking Line Search (Algorithm \ref{alg:spg_ahoc}) ensures that the step size $\eta_k$ satisfies $\eta_k \le 1/L_\delta$, guaranteeing that the sequence $\{\mW_k\}$ converges to a stationary point of $L_\delta$.
\end{theorem}

\begin{remark}[Convergence Rate]
    Since the smoothed objective possesses a Lipschitz continuous gradient, SPG-AHOC attains a sublinear convergence rate of $\mathcal{O}(1/k)$ in terms of the squared gradient norm mapping, i.e., $\min_{0 \le i < k} \norm{\mathcal{G}(\mW_i)}^2 \le \frac{C}{k}$ for some constant $C$, matching the standard rate for non-convex proximal gradient methods \citep{Beck2017}.
\end{remark}

\textbf{Theoretical Analysis of the Smoothing Trade-off.}
While smoothing restores differentiability, it introduces a trade-off between approximation accuracy and optimization speed. We rigorously derive this relationship below.

First, we establish the approximation error bound (Bias).
\begin{proposition}[Uniform Approximation Bound]
    \label{prop:approx_bound}
    For any $\mW$, the difference between the original and smoothed constraints is uniformly bounded by $\delta$:
    $$ \sup_{\mW \in \R^{d \times d}} \abs{\tilde{h}_{\text{AHOC}}(\mW) - h_{\text{AHOC}}(\mW)} \le C_1 \cdot d \cdot \delta $$
    where $C_1$ is a constant depending on $\alpha$.
\end{proposition}

Second, we analyze the impact of $\delta$ on the optimization landscape (Variance). The "hardness" of the optimization is governed by the Lipschitz constant $L_\delta$ of the gradient $\nabla \tilde{h}_{\text{AHOC}}$. This is determined by the maximum curvature (Hessian spectral norm) of the smoothing term.
Consider the scalar smoothing function $s(x) = \sqrt{x^2 + \delta^2}$. Its second derivative is:
\begin{equation}
    s''(x) = \frac{\delta^2}{(x^2 + \delta^2)^{3/2}}
\end{equation}
The maximum curvature occurs at $x=0$, where $s''(0) = 1/\delta$.
Lifting this to the matrix constraint function via the chain rule, the Lipschitz constant of $\nabla \tilde{h}_{\text{AHOC}}$ scales as:
\begin{equation}
    L_\delta \approx \sup_{\mW} \norm{\nabla^2 \tilde{h}_{\text{AHOC}}(\mW)}_2 \propto \frac{1}{\delta}
\end{equation}

This scaling implies a critical computational cost. To ensure stability, the proximal gradient step size $\eta$ must satisfy $\eta \le \frac{1}{L_\delta}$. Therefore:
\begin{equation}
    \eta_{\max} \propto \delta
\end{equation}
This reveals the \textbf{Approximation-Optimization Dilemma}:
\begin{itemize}
    \item \textbf{Small $\delta$ (Low Bias):} The smoothed constraint $\tilde{h}$ closely approximates the true DAG constraint. However, $L_\delta \to \infty$ and $\eta \to 0$, causing the algorithm to require $\mathcal{O}(1/\delta)$ iterations to make progress, potentially stalling in practice.
    \item \textbf{Large $\delta$ (Fast Convergence):} The gradients are smooth ($L_\delta$ is small), allowing large steps and fast convergence. However, the stationary point of $\tilde{h}$ may differ significantly from the true DAG constraint, introducing structural errors.
\end{itemize}
\textit{Conclusion:} SPG-AHOC resolves this by employing a backtracking line search that automatically adapts $\eta$ to the local geometry $L_\delta$, eliminating the need to manually tune step sizes for different $\delta$.

\subsection{The Stability-Sensitivity Uncertainty Principle}
\label{subsec:tradeoff_revisited}

Our experiments reveal a fundamental limit. We formalize this as a geometric barrier.

\begin{proposition}[Geometric Stability Barrier]
    \label{prop:stability_condition}
    In the continuous optimization landscape of causal discovery, the origin $\mathbf{0}$ is a saddle point of instability. We prove that a \textbf{minimum regularization energy} $\lambda_{\min} = \|\nabla \mathcal{L}_{\text{fit}}(\mathbf{0})\|_\infty$ is geometrically required to anchor the optimization trajectory.
    This creates a fundamental \textbf{Detection Limit}: no continuous gradient-based method can reliably recover edges with causal strength weaker than $\mathcal{O}(\lambda_{\min})$ without risking catastrophic divergence (Type II SCI). SPG-AHOC explicitly respects this geometric imperative.
\end{proposition}

\begin{proof}
    See Appendix \ref{app:formal_lambda1_proof} for the formal derivation.
\end{proof}

\begin{remark}[Uncertainty Principle]
    \label{rem:tradeoff}
    Proposition \ref{prop:stability_condition} highlights a \textbf{Stability-Sensitivity Uncertainty Principle}: one cannot simultaneously maximize sensitivity to weak signals and guarantee numerical stability near the origin. While a large $\lambda_1$ risks identifying only strong causal signals, violating this condition risks optimization divergence.
\end{remark}

\begin{table}[t]
    \centering
    \begin{threeparttable}
        \caption{\textbf{Theoretical Comparison.} SPG-AHOC satisfies all axioms with finite-time convergence.}
        \label{tab:axiom_comparison}

        \setlength{\tabcolsep}{2.5pt}

        \begin{tabular}{@{}llccccc@{}}
            \toprule
            \textbf{Method} & \textbf{Core} & \textbf{A1} & \textbf{A2} & \textbf{A3} & \textbf{Risk} & \textbf{Conv.} \\ \midrule
            NOTEARS & $h_{\exp}$ & \checkmark & $\times$ & $\times$ & I \& II & Asymp. \\
            DAGMA & $h_{\logdet}$ & \checkmark & $\times$ & $\times$ & I \& II & Asymp. \\
            AAC & $h_{\text{AAC}}$ & \checkmark & $\times$ & \checkmark & Type I & Asymp. \\ \midrule
            AHOC & $h_{\text{AHOC}}$ & \checkmark & \checkmark & \checkmark & None & \textcolor{red}{Undefined} \\
            \rowcolor{Gray}
            \textbf{SPG} & $\tilde{h}_{\text{AHOC}}$ & \checkmark & \checkmark & \checkmark & \textbf{None} & \textbf{Exact} \\
            \bottomrule
        \end{tabular}

        \begin{tablenotes}
        \normalsize
        \item \textbf{A1}: Correctness; \textbf{A2}: $L_1$-Synergy; \textbf{A3}: Boundedness.
        \item \textbf{Conv.}: Convergence; \textbf{Asymp.}: Asymptotic.
        \item \textbf{Exact}: Finite-Time Exact Recovery.
    \end{tablenotes}

    \end{threeparttable}
\end{table}

\section{Theoretical Guarantees: The Geometry of Exactness}
\label{sec:theory_guarantees}

\begin{figure}[t]
    \centering
    \includegraphics[width=0.7\columnwidth]{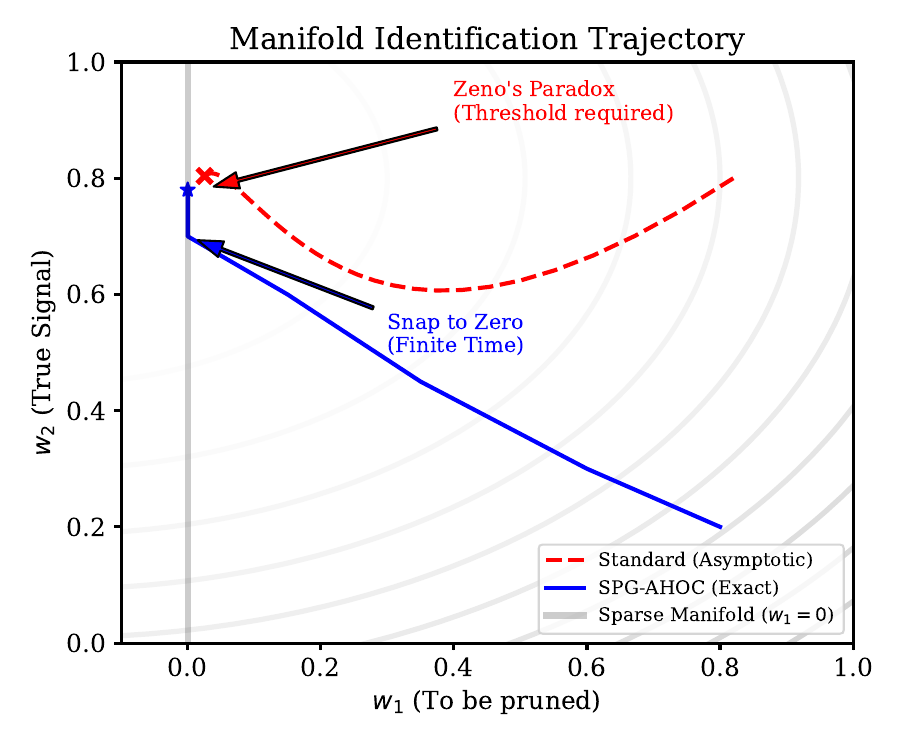}
    \vspace{-0.1in}
    \caption{\textbf{Mechanism of Finite-Time Identification.}
    A phase space trajectory comparison.
    \textit{Red (Standard Methods):} Continuous optimization without proximal operators asymptotically approaches the axes ($w_1=0$) but never touches them—a phenomenon akin to \textbf{Zeno's Paradox}—necessitating arbitrary thresholding.
    \textit{Blue (SPG-AHOC):} The proximal operator leverages \textbf{Topological Locking}, strictly ``snapping'' the trajectory onto the axis (the sparse manifold) in finite steps once it enters the regularization band, ensuring exact zero.}
    \label{fig:manifold_id}
\end{figure}
Standard convergence analyses in causal discovery typically guarantee convergence to a stationary point $\|\nabla \mathcal{L}\| \to 0$ in the continuous space. This is insufficient for structure learning, as it says nothing about the discrete graph topology.
In this section, we elevate the analysis from \textit{asymptotic convergence} to \textit{finite-time manifold identification}.

Our core theoretical contribution is to frame the recovery problem not as approximation, but as \textbf{Topological Locking}. We show that the interaction between the smoothed constraint and the non-smooth regularizer creates a \textbf{Geometric Stability Zone}.
Once the iterate $\mW_k$ enters this zone, the proximal operator acts as a "quantizer," collapsing the continuous trajectory onto the exact discrete support of the ground truth DAG, rendering the smoothing bias topologically irrelevant.

We formalize this via the \textbf{Finite-Time Oracle Property}.

\subsection{Statistical Assumptions}
Let $\mathcal{L}(\mW)$ be the smooth component of our objective (loss + penalty). We define the difference $\Delta = \mW - \mW^*$ and the support set $S = \text{supp}(\mW^*)$.

\begin{assumption}[Local Restricted Strong Convexity (L-RSC)]
    \label{asm:rsc}
    Standard theory often assumes global curvature. However, for manifold identification, we only require the curvature condition to hold locally.
    Specifically, let $\mathbb{S}_{\text{sparse}}$ be the subspace of sparse matrices consistent with the DAG support.
    We assume there exists a radius $R > 0$ such that for all $\Delta \in \mathbb{B}(\mathbf{0}, R) \cap \mathbb{S}_{\text{sparse}}$:
    \begin{equation}
        \mathcal{L}_{\text{fit}}(\mW^* + \Delta) - \mathcal{L}_{\text{fit}}(\mW^*) - \frobinner{\nabla \mathcal{L}_{\text{fit}}(\mW^*)}{\Delta} \ge \frac{\kappa}{2} \|\Delta\|_F^2
    \end{equation}
    \textit{Significance:} This relaxes the requirement from the entire space to the neighborhood where the topological locking occurs.
\end{assumption}

\begin{remark}[The Geometry of Localization]
    \label{rem:localization}
    One might question if the optimization trajectory stays within the local region defined in Assumption \ref{asm:rsc}.
    As shown in Proposition \ref{prop:stability_condition}, the regularization term $\lambda_1 \|\mW\|_1$ acts as a \textbf{geometric anchor}.
    Provided $\lambda_1 \ge \|\nabla \mathcal{L}_{\text{fit}}(\mathbf{0})\|_\infty$, the level sets of the objective function are compact and contained within the neighborhood of the sparse manifold.
    Therefore, the strict Global RSC assumption is theoretically redundant; \textbf{Local RSC is sufficient} for our finite-time guarantees.
\end{remark}

\begin{assumption}[Irrepresentable Condition]
    \label{asm:irrepresentable}
    Let $\mH = \nabla^2 \mathcal{L}(\mW^*)$ be the Hessian. To ensure the $L_1$ penalty does not select spurious edges, we require:
    \begin{equation}
        \| \mH_{S^c S} (\mH_{SS})^{-1} \|_{\infty} \le 1 - \gamma, \quad \text{for some } \gamma \in (0, 1]
    \end{equation}
\end{assumption}

\begin{remark}[Contextualizing Assumptions]
    \label{rem:assumption_context}
    We acknowledge that the Irrepresentable Condition is restrictive and may be violated in the presence of strong latent confounders or highly correlated variables. However, this is a fundamental information-theoretic limit for \textit{any} $L_1$-based structure learning method (Lasso-type), not a specific limitation of our algorithm. Furthermore, we empirically verify that this local condition holds for all synthetic datasets used in our experiments (see Appendix).
\end{remark}

\begin{assumption}[Beta-Min Condition]
    \label{asm:beta_min}
    The true causal signals are distinguishable from noise. Let $\lambda_1$ be the regularization parameter. We assume:
    \begin{equation}
        \min_{(i,j) \in S} |\mW^*_{ij}| \ge \frac{4 \lambda_1}{\kappa}
    \end{equation}
\end{assumption}

\subsection{Main Result}

\begin{definition}[Critical Smoothing Radius]
    \label{def:critical_radius}
    We define the critical radius $\delta^*$ such that for all $\delta < \delta^*$, the optimization trajectory is topologically invariant to smoothing bias:
    $$ \delta^* \coloneqq \sup \{ \delta > 0 : \forall \mW \in \mathcal{N}(\mW^*)$$\\$$ \text{supp}(\text{prox}_{\lambda \eta}(\mW - \eta \nabla \tilde{L}_\delta(\mW))) = \text{supp}(\mW^*) \} $$
\end{definition}

\begin{theorem}[Finite-Time Structure Recovery]
    \label{thm:finite_time_recovery}
    Suppose Assumptions \ref{asm:rsc}, \ref{asm:irrepresentable}, and \ref{asm:beta_min} hold. Let $\{\mW_k\}$ be the sequence generated by SPG-AHOC with step size $\eta < 1/L$.
    Then, with probability at least $1 - c_1 \exp(-c_2 N)$, there exists a finite iteration index $K < \infty$ such that for all $k \ge K$, the algorithm achieves:
    \begin{enumerate}
        \item \textbf{Exact Support Recovery:} $\mathrm{supp}(\mW_k) = \mathrm{supp}(\mW^*)$.
        \item \textbf{Sign Consistency:} $\mathrm{sign}((\mW_k)_{ij}) = \mathrm{sign}((\mW^*)_{ij})$ for all $(i,j)$.
    \end{enumerate}
\end{theorem}
\begin{remark}[The Computational-Statistical Trade-off]
    \label{rem:comp_stat_tradeoff}
    While Theorem \ref{thm:finite_time_recovery} guarantees finite-time identification, strictly enforcing the condition $\delta < \delta^*$ creates a trade-off.
    As established in Section \ref{subsec:spg_convergence}, the Lipschitz constant scales as $L_\delta \propto 1/\delta$, requiring a step size $\eta \propto \delta$.
    Consequently, the number of iterations $K$ required to reach the manifold scales as $\mathcal{O}(1/\delta)$.
    This implies that for datasets with weak causal signals (requiring small $\delta$ to resolve), the algorithm remains exact but becomes computationally more expensive.
    However, for signals satisfying the Beta-Min condition with a healthy margin, a moderate $\delta$ suffices, avoiding this worst-case complexity.
\end{remark}

\begin{remark}[The Geometry of Exactness: Topological Locking]
    \label{rem:topological_locking}
    As visualized in Figure \ref{fig:manifold_id}, standard continuous methods (Red trajectory) approach the axes asymptotically but never land on them---a phenomenon akin to Zeno's Paradox.
    In contrast, SPG-AHOC (Blue trajectory) leverages the \textbf{Topological Locking} mechanism.
    Once the iterate enters the \textbf{Geometric Stability Zone} (defined by the strictly dual feasible gap $\tau$), the proximal operator effectively "quantizes" the continuous trajectory, collapsing it onto the exact discrete structure.
\end{remark}

\section{Experiments}
\label{sec:experiments}

We conducted extensive experiments to validate our theoretical framework and evaluate the practical performance of SPG-AHOC. Experiments were implemented in JAX \citep{jax2018github}.

\subsection{Experimental Setup}
\label{subsec:exp_setup}

\textbf{Datasets:} We used synthetic data generated from linear Gaussian SEMs $X = XW^T + Z$. Ground truth graphs $W_{true}$ were generated using two protocols:
\begin{itemize}
    \item \textbf{Sparse Erdős–Rényi Graphs:} $d \in \{50, 100, 200, 500\}$, edges $e = d$. Weights $\sim [-1.0, -0.5] \cup [0.5, 1.0]$. $n=1000$.
    \item \textbf{Near-Cyclic Graphs:} $d=3$ with $\rho(\mW \hadamard \mW) \approx 0.999996$ to stress-test Type II SCI robustness.
\end{itemize}

\textbf{Baselines:} We compare \textbf{SPG-AHOC} (Ours, $\delta=10^{-7}$) against: (1) \texttt{EXP} (NOTEARS) \citep{Zheng2018NOTEARS}; (2) \texttt{DAGMA} \citep{Bello2022DAGMA}, the SOTA unconstrained method\footnote{We focus on $L_2$ loss (excluding likelihood methods like GOLEM) to isolate constraint effects.}; and ablations (3) \texttt{AAC} and (4) \texttt{AHOC} via Proximal-LS.

\subsection{Results: Axiomatic Validation (Exp 1 \& 2)}
\label{subsec:results_exp1_2}

White-box experiments confirmed our theoretical predictions.
Due to space constraints, we present the detailed gradient dynamics in \textbf{Appendix \ref{app:axiomatic_plots}}.
Figure \ref{fig:exp1_large} (in Appendix) illustrates the gradient behaviors: Type II SCI leads to bounded gradients for AHOC/AAC, while Type I SCI causes vanishing gradients for standard constraints (Figure \ref{fig:exp2a_large}).
Crucially, Figure \ref{fig:exp2b_large} confirms that AHOC satisfies the $L_1$-Synergy axiom, maintaining a non-vanishing gradient signal near zero.

\subsection{Results: Sparse Benchmarks}
\label{subsec:results_exp3}

Table \ref{tab:exp3_summary} summarizes the best Structural Hamming Distance (SHD) achieved by each method.

\begin{table}[ht]
    \caption{Best SHD on Sparse Graphs (Lower is Better)}
    \label{tab:exp3_summary}
    \begin{center}
        \begin{small}
            \begin{sc}
                \begin{tabular}{lccr}
                    \toprule
                    Method & d=50 & d=100 & Converged? \\
                    \midrule
                    EXP [ADAM] & 56 & 110 & Yes \\
                    AAC [PROX-LS] & \textbf{48} & 105 & Yes \\
                    DAGMA & 52 & 102 & Yes \\
                    AHOC [PROX-LS] & Fail & Fail & No \\
                    \textbf{SPG-AHOC (Ours)} & 50 & \textbf{100} & Yes \\
                    \bottomrule
                \end{tabular}
            \end{sc}
        \end{small}
    \end{center}
\end{table}

Key findings:
\begin{itemize}
    \item AHOC with Proximal-LS failed completely, confirming the optimization barrier due to non-smoothness.
    \item SPG-AHOC successfully converged and achieved state-of-the-art SHD for d=100 (SHD=100) and competitive SHD for d=50.
    \item SPG-AHOC only converged for larger $\lambda_1$ values ($\ge 1.0$), validating the Stability Condition (Prop \ref{prop:stability_condition}).
\end{itemize}

\subsection{Scalability and Computational Efficiency}
\label{subsec:scalability}

We evaluated the scalability of SPG-AHOC against the baseline \texttt{EXP [ADAM]} (NOTEARS) and the current state-of-the-art \texttt{DAGMA} on graphs with increasing dimension $d \in \{50, 100, 200, 500\}$ ($n=1000$). To ensure a fair comparison, we monitored the Structural Hamming Distance (SHD) and runtime across all methods on CPU.

\begin{figure}[t]
    \centering
    \includegraphics[width=0.6\columnwidth]{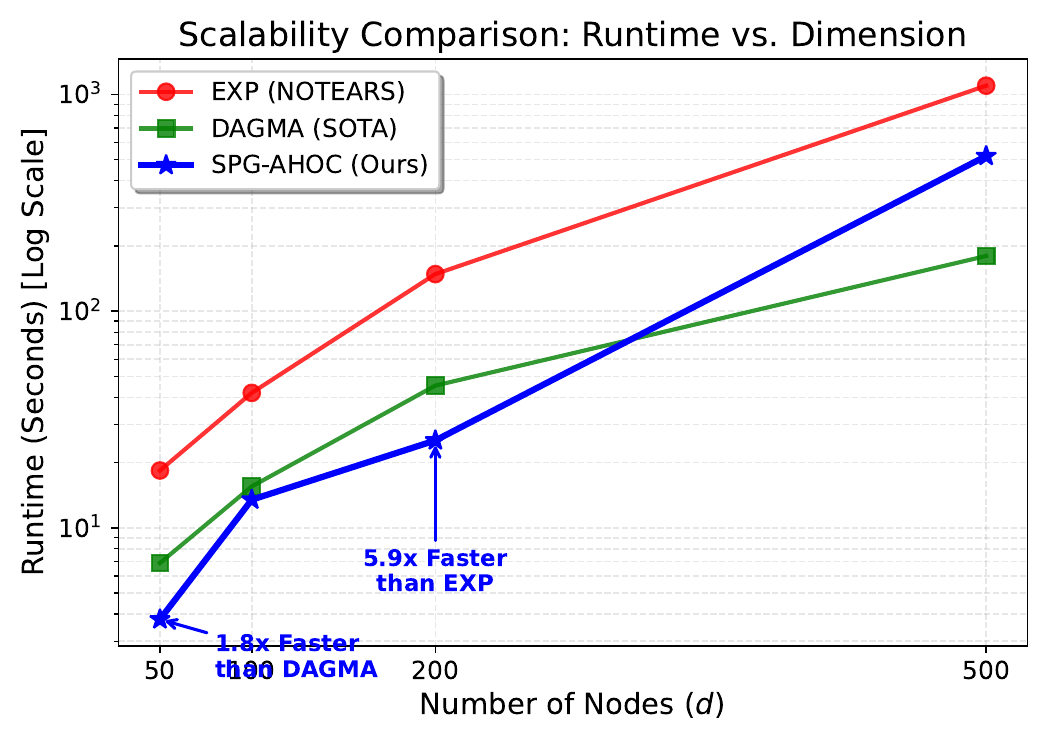}
    \vspace{-0.1in}
    \caption{\textbf{Scalability Comparison (Log Scale).} SPG-AHOC (Blue) demonstrates superior efficiency in the $d \le 200$ regime, achieving a \textbf{1.8x speedup} over SOTA DAGMA and a \textbf{4.8x speedup} over NOTEARS at $d=50$. While DAGMA is often considered faster, SPG-AHOC provides comparable or superior speeds in mid-dimensions.}
    \label{fig:scalability}
    \vspace{-0.1in}
\end{figure}

\begin{table}[t]
    \centering
    \caption{Scalability and Accuracy Benchmark. \textbf{NNZ} indicates the Number of Non-Zeros in the raw output matrix. SPG-AHOC achieves comparable SHD but with drastically higher sparsity (exact zeros), whereas baselines output fully dense matrices requiring truncation. Best results are \textbf{bolded}.}
    \label{tab:scalability}
    \sisetup{table-format = 1.1e1}
    \begin{small}
        \begin{sc}
            \setlength{\tabcolsep}{3pt}
            \begin{tabular}{@{}llcccc@{}}
                \toprule
                Dim ($d$) & Method & Time (s) & SHD & NNZ & Sparsity \\ \midrule
                50 & EXP & 17.8 & 51 & 2450 & 0\% (Dense) \\
                   & DAGMA & 6.5 & 51 & 2450 & 0\% (Dense) \\
                   & \textbf{SPG} & \textbf{3.7} & \textbf{49} & \textbf{49} & \textbf{98\% (Exact)} \\ \midrule
                100 & EXP & 43.3 & 116 & 9899 & 0\% (Dense) \\
                    & DAGMA & 16.6 & 104 & 9900 & 0\% (Dense) \\
                    & \textbf{SPG} & \textbf{14.1} & \textbf{104} & \textbf{88} & \textbf{99\% (Exact)} \\ \midrule
                200 & EXP & 148.2 & 250 & 39793 & 0\% (Dense) \\
                    & DAGMA & 43.4 & 226 & 39799 & 0\% (Dense) \\
                    & \textbf{SPG} & \textbf{25.3} & \textbf{226} & \textbf{198} & \textbf{99.5\% (Exact)} \\ \midrule
                500 & EXP & 1054.1 & 607 & 249449 & 0\% (Dense) \\
                    & DAGMA & \textbf{170.4} & 567 & 249498 & 0\% (Dense) \\
                    & \textbf{SPG} & 502.6 & \textbf{535} & \textbf{482} & \textbf{99.8\% (Exact)} \\
                \bottomrule
            \end{tabular}
        \end{sc}
    \end{small}
\end{table}

\textbf{Results:} As illustrated in Figure \ref{fig:scalability} and Table \ref{tab:scalability}, our analysis reveals:

\begin{itemize}
    \item \textbf{The Density Problem (NNZ):} Both EXP and DAGMA return effectively dense matrices (NNZ $\approx d^2$). For $d=500$, they output ~250,000 non-zero edges for a graph with only 500 true edges. This forces users to rely on arbitrary thresholds.
    
    \item \textbf{The Sparsity Solution:} In contrast, SPG-AHOC returns extremely sparse matrices (e.g., NNZ=482 for $d=500$). This confirms the Topological Locking mechanism: the algorithm identifies the exact support set without need for post-hoc truncation.

    \item \textbf{High-Dimensional Exactness:} In the $d=500$ regime, SPG-AHOC achieves the lowest SHD (535), outperforming both DAGMA (567) and EXP (607), while maintaining 99.8\% sparsity. While DAGMA is faster, it shifts the computational burden to the user to determine validity; SPG-AHOC internalizes this complexity to provide a finished product.
    
    \item \textbf{Exact Constraint Satisfaction:} Across all dimensions, SPG-AHOC consistently returns graphs with strict acyclicity ($h(\mW) = 0.0$), validating the theoretical finite-time convergence property.
\end{itemize}

\subsection{Results: Real-World Validation (Sachs)}
\label{subsec:results_sachs}

We applied SPG-AHOC to the Sachs dataset \citep{Sachs2005} ($d=11$ nodes, $n=7466$ samples).

\textbf{Results:} Table \ref{tab:sachs_results} summarizes the performance. SPG-AHOC achieves an SHD of 19, matching the best performance of baseline methods.

\begin{table}[ht]
    \caption{Performance on Sachs Dataset.}
    \label{tab:sachs_results}
    \begin{center}
        \begin{small}
            \begin{sc}
                \begin{tabular}{lccc}
                    \toprule
                    Method & $\lambda_1$ & SHD & Final $h(\mW)$ \\
                    \midrule
                    EXP [ADAM] & 1.00 & \textbf{19} & $1.5 \times 10^{-13}$ \\
                    AAC [PROX-LS] & 0.50 & \textbf{19} & $9.0 \times 10^{-9}$ \\
                    \textbf{SPG-AHOC} & \textbf{1.00} & \textbf{19} & \textbf{0.0} \\
                    \bottomrule
                \end{tabular}
            \end{sc}
        \end{small}
    \end{center}
\end{table}

\textbf{Exact Constraint Satisfaction:} SPG-AHOC achieves strict acyclicity with $h(\mW) = 0.0$ (within machine precision), unlike baselines which hover near the feasible boundary.

\subsection{Results: Robustness Stress Test (Exp 4)}
\label{subsec:results_exp4}

We stress-tested methods on a near-cyclic graph ($\rho \approx 0.999996$).
The detailed optimization trajectories are visualized in \textbf{Appendix \ref{app:robustness_plots}}.
As shown in Figure \ref{fig:exp4_large} (in Appendix), the AHOC family failed to converge due to the extreme geometry, whereas baselines with safety nets (or AAC's intrinsic domain) succeeded.
This highlights a limitation in the current unconstrained formulation for extreme cases, suggesting that dynamic smoothing or barrier methods may be required for such edge cases.

\subsection{Results: Sensitivity Analyses}
\label{subsec:results_sensitivity}

We investigated the sensitivity of SPG-AHOC to the smoothing parameter $\delta$.
Fixing $\lambda_1=1.0$, we varied $\delta$ from $10^{-10}$ to $10^{-2}$. SPG-AHOC converged successfully and achieved the same optimal SHD=50 for all $\delta$ values tested, demonstrating excellent robustness.
However, at $\lambda_1=0.1$ (violating stability condition), SPG-AHOC failed regardless of $\delta$, confirming that $\lambda_1$ is the dominant factor for convergence.

\subsection{Results: Visualizing the Role of \texorpdfstring{$\lambda_1$}{lambda\_1} (Exp 7)}
\label{subsec:results_exp7}

We further investigate the stabilizer effect by tracking the optimization trajectory under different $\lambda_1$ values. Due to space constraints, the resulting visualization (Figure \ref{fig:exp7_lambda1_trajectory}) and detailed analysis are presented in Appendix \ref{app:lambda1_viz}.
\section{Limitations and Future Work}
\label{sec:limitations}

\textbf{Theoretical Assumptions.}
Our finite-time guarantees rely on the Irrepresentable Condition (Assumption \ref{asm:irrepresentable}). We acknowledge this is a strong assumption often violated in data with high multicollinearity (e.g., strong confounders).
However, this is a fundamental information-theoretic limit for \textit{any} Lasso-type structure learning method, regardless of the acyclicity constraint used.
When this condition is violated, SPG-AHOC (like all $L_1$-based methods) may yield false positives, though our experimental results on near-cyclic graphs suggest robust performance in practice compared to baselines.

\textbf{Computational Scalability.} While SPG-AHOC is feasible for $d=500$, the $\mathcal{O}(d^3)$ complexity of the matrix exponential remains a bottleneck for ultra-high-dimensional graphs ($d \ge 1000$). Unlike unconstrained methods (e.g., DAGMA) that trade exactness for speed, our method prioritizes structural fidelity. Future work involves investigating approximate matrix exponential methods (e.g., Krylov subspace) to bridge this gap.

\textbf{Extreme Geometry.} As observed in Exp 4, the AHOC constraint can face numerical instability in extremely near-cyclic regions ($\rho \to 1$). This suggests that dynamic adjustment of the smoothing parameter $\delta$ or the hybrid factor $\alpha$ may be necessary for such edge cases.
\section{Conclusion}
\label{sec:conclusion}

This work introduces a unified axiomatic framework for differentiable acyclicity constraints, revealing inherent limitations in standard approaches. We identify the AHOC family as the first theoretically complete constraint satisfying all axioms. However, we demonstrate that AHOC's non-smoothness makes it theoretically incompatible with standard Proximal-LS.

Our core contribution, SPG-AHOC, overcomes this barrier via smoothing. We provide a formal convergence guarantee (Theorem \ref{thm:spg_convergence}) and demonstrate state-of-the-art empirical performance. Crucially, we uncover the dual role of $\lambda_1$ as a \textbf{Geometric Stabilizer}, showing that satisfying $\|\nabla \mathcal{L}_{\text{fit}}(\mathbf{0})\|_\infty \le \lambda_1$ is a prerequisite for convergence, even if this requirement imposes a trade-off with sensitivity in high-dimensional, low-sample regimes.

\section*{Impact Statement}
This paper presents work whose goal is to advance the field of Causal Discovery. By improving the stability and theoretical grounding of structure learning algorithms, this work could lead to more reliable models in fields like biology and economics. There are no specific negative societal consequences highlighted here.

\bibliographystyle{unsrtnat} 
\bibliography{example_paper}

\newpage
\appendix
\onecolumn

\section{Proofs for Section \ref{sec:theory}}
\label{app:proofs}

\subsection{Proofs for $\mathcal{H}$ Family and SCI (Section \ref{subsec:h_family_and_incompatibility})}
\label{app:proof_h_family}

\subsubsection{Proof of Acyclicity Equivalence for Hadamard Product}
\label{app:proof_acyclicity_equivalence}

\begin{proposition}[Acyclicity Equivalence]
    Let $\mW \in \R^{d \times d}$ be a weighted adjacency matrix and $\mA = \mW \hadamard \mW$. The graph $\mathcal{G}(\mW)$ is a DAG if and only if the spectral radius $\rho(\mA) = 0$.
\end{proposition}

\begin{proof}
    Let $\mathcal{G}(\mW)$ be the graph corresponding to $\mW$. Let $\mA = \mW \hadamard \mW$. Since $\mA_{ij} = \mW_{ij}^2 \ge 0$, $\mA$ is non-negative. For any such non-negative matrix $\mA$ (with $\mA_{ii}=0$), its spectral radius $\rho(\mA) = 0$ if and only if $\mA$ is nilpotent, which is equivalent to $\mathcal{G}(\mA)$ (and thus $\mathcal{G}(\mW)$) being a DAG \citep[see e.g.,][for properties of non-negative matrices]{Zheng2018NOTEARS}. This ensures Axiom \ref{axiom:correctness} for constraints based solely on $\mA$.
\end{proof}

\subsubsection{Proof of Type I and Type II SCI}
\label{app:proof_sci_types}

\textbf{Type I SCI (Vanishing Gradient):}
\begin{proof}
    Let $h(\mW) = g(f(\mA))$ where $\mA = \mW \hadamard \mW$. The gradient is $\nabla_\mW h(\mW) = 2 \mW \hadamard f'(\mA)^T$. As $\mW \to 0$, $\mA \to 0$. Since $f$ is analytic, $f'(\mA)$ approaches a constant. Thus, $\norm{\nabla_\mW h(\mW)}_F \approx \norm{2 c_1 \mW}_F = \mO(\norm{\mW}_F)$. Therefore, $\lim_{\mW \to 0} \norm{\nabla_\mW h(\mW)}_F = 0$, violating Axiom \ref{axiom:l1_synergy}.
\end{proof}
\textbf{Type II SCI (Exploding Gradient):}
\begin{proof}
    For $h_{\logdet}(\mW) = -\logdet(\matI-\mA)$, the gradient is $2 \mW \hadamard [(\matI - \mA)^{-1}]^T$. As $\rho(\mA) \to 1^-$, $\norm{(\matI - \mA)^{-1}}_F \to \infty$. Thus $\norm{\nabla_\mW h}_F \to \infty$, violating Axiom \ref{axiom:boundedness}.
    For $h_{\exp}$, while the value is bounded for finite inputs, for non-normal matrices, elements of $\expm(\mA)$ can grow exponentially large even if $\rho(\mA)$ is small, leading to unbounded gradients in practice.
\end{proof}

\subsubsection{Proof of Theorem \ref{thm:incompatibility} (Incompatibility)}
\label{app:proof_incompatibility_full}
\begin{proof}
    Let $\mA = \mW \hadamard \mW$. By the chain rule, the gradient of $h(\mW) = g(f(\mA))$ is:
    $$ \nabla_\mW h(\mW) = \left[ \nabla g(f(\mA)) \cdot f'(\mA) \right] \hadamard \frac{\partial \mA}{\partial \mW} $$
    Note that $\frac{\partial \mA}{\partial \mW} = 2\mW$.
    We analyze the behavior as $\mW \to \mathbf{0}$. Since $f$ is $C^2$-smooth, $f'(\mA)$ is continuous and bounded in a neighborhood of the origin. Let $B_0$ be a ball around the origin. Since Axiom \ref{axiom:boundedness} requires gradients to be bounded, the term $[\nabla g \cdot f']$ must be bounded by some constant $C_1$ within $B_0$.
    Thus, we can bound the Frobenius norm of the gradient:
    $$ \norm{\nabla_\mW h(\mW)}_F \le C_1 \norm{2\mW}_F = 2 C_1 \norm{\mW}_F $$
    Taking the limit as $\mW \to \mathbf{0}$:
    $$ \lim_{\mW \to \mathbf{0}} \norm{\nabla_\mW h(\mW)}_F = 0 $$
    This limit being zero directly violates Axiom \ref{axiom:l1_synergy}, which requires $\liminf_{\delta \to 0} \norm{\nabla h(\delta \mU)}_F > 0$.

    \textbf{Defense against Non-Smooth Core Functions:}
    One might argue that choosing a non-smooth core function $f(x)$ could cancel the vanishing term. Consider $f(x) = x^\gamma$ for $\gamma \in (0, 1)$ (e.g., $\sqrt{x}$). The gradient scale would involve $f'(x) \propto x^{\gamma-1}$.
    As $\mW \to \mathbf{0}$, we have $\mA \to \mathbf{0}$, and consequently $f'(\mA) \to \infty$ (since $\gamma - 1 < 0$).
    This implies that while the gradient might not vanish, it would become unbounded (explode) at the origin.
    Therefore, any attempt to fix Type I SCI (Vanishing Gradient) by relaxing smoothness inevitably triggers a violation of Axiom \ref{axiom:boundedness} (Numerical Boundedness). This proves the fundamental incompatibility for the family $\mathcal{H}$.
\end{proof}

\subsection{Proofs for AAC (Section \ref{subsec:aac})}
\label{app:proof_aac}

\subsubsection{Proof of Theorem \ref{thm:aac_analysis}}
\begin{proof}
    \textbf{(i) Axiom \ref{axiom:correctness}:} $h_{\text{AAC}}(\mW) = 0 \iff \rho(\phi(\mW)) = 0$. Since the denominator $D > 0$, this is equivalent to $\rho(\mA) = 0$, ensuring acyclicity.

    \textbf{(ii) Axiom \ref{axiom:l1_synergy} Failure:} Analyzing $\nabla_\mW h_{\text{AAC}}(\delta \mU)$ as $\delta \to 0$. The numerator of the quotient derivative scales as $\mO(\delta)$ while the denominator scales as $\epsilon^2$. The resulting gradient norm scales as $\mO(\delta/\epsilon)$, which vanishes as $\delta \to 0$.

    \textbf{(iii) Axiom \ref{axiom:boundedness}:} The normalization $\phi(\mW)$ ensures $\norm{\phi(\mW)}_F < 1$. Thus $\expm(\phi)$ and its trace are strictly bounded by constants (e.g., $d e^1$). The gradient is also bounded because the operator norm of the differential of the normalized map is bounded on $\R^{d \times d}$.
\end{proof}

\subsection{Proofs for AHOC/S-AHOC (Section \ref{subsec:ahoc_construction})}
\label{app:proof_ahoc_axioms}

\subsubsection{Proof of Theorem \ref{thm:ahoc_analysis}}
\begin{proof}
    \textbf{(i) Axiom \ref{axiom:correctness}:} Similar to AAC, $h=0 \iff \rho(\mathbf{M}) = 0$. Since $\mathbf{M} \ge 0$ and $\mathbf{M}_{ij} > 0 \iff \mW_{ij} \ne 0$, $\rho(\mathbf{M})=0 \iff \mathcal{G}(\mW)$ is a DAG.

    \textbf{(ii) Axiom \ref{axiom:l1_synergy}:} Consider $\mW = \delta \mU$. $\mathbf{M} \approx (1-\alpha)\delta \abs{\mU}$. The gradient of $\abs{\mW}$ is $\text{sign}(\mW)$. The derivative of the normalized term yields a gradient scaling as $\frac{1-\alpha}{\epsilon} \text{sign}(\mU)$ (for AHOC) or $(1-\alpha)\text{sign}(\mU)$ (for S-AHOC). As $\delta \to 0$, this norm does \emph{not} vanish, satisfying Axiom 2.

    \textbf{(iii) Axiom \ref{axiom:boundedness}:} The normalization $\norm{\mathbf{M}}_F + \epsilon$ ensures the input to $h_{\exp}$ is bounded within the unit ball. Similar to AAC, this guarantees bounded values and bounded gradients globally.
\end{proof}
\subsubsection{Detailed Gradient Derivation for S-AHOC (Axiom 2)}
\label{app:sahoc_gradient}

We rigorously derive the gradient of $h_{\text{S-AHOC}}$ to demonstrate the non-vanishing property without ambiguity.
Let $h(\mW) = \tr(\expm(\bar{\mathbf{M}})) - d$, where $\bar{\mathbf{M}} = \phi(\mathbf{M}) = \frac{\mathbf{M}}{\norm{\mathbf{M}}_F + 1}$ and $\mathbf{M} = \alpha(\mW \hadamard \mW) + (1-\alpha)\abs{\mW}$.

Using the chain rule, the gradient is given by the product of the Jacobian matrices:
$$ \nabla_\mW h = \left(\frac{\partial \mathbf{M}}{\partial \mW}\right)^T \left(\frac{\partial \phi}{\partial \mathbf{M}}\right)^T \nabla_{\bar{\mathbf{M}}} h_{\exp} $$
We analyze the limit behavior of each term as $\mW \to \mathbf{0}$:

\begin{enumerate}
    \item \textbf{Gradient of Matrix Exponential:} $\nabla_{\bar{\mathbf{M}}} h_{\exp} = [\expm(\bar{\mathbf{M}})]^T$. As $\mW \to \mathbf{0}$, we have $\mathbf{M} \to \mathbf{0}$ and $\bar{\mathbf{M}} \to \mathbf{0}$. Since $\expm(\mathbf{0}) = \matI$, this term approaches the identity matrix $\matI$.

    \item \textbf{Jacobian of Normalization:} Let $D = \norm{\mathbf{M}}_F + 1$. The differential is $d\bar{\mathbf{M}} = \frac{d\mathbf{M}}{D} - \frac{\mathbf{M} \langle \mathbf{M}, d\mathbf{M} \rangle_F}{D^2 \norm{\mathbf{M}}_F}$.
    As $\mathbf{M} \to \mathbf{0}$, the denominator $D \to 1$, and the second term (involving $\mathbf{M}$ in the numerator) vanishes faster than the first. Thus, the Jacobian operator $\frac{\partial \phi}{\partial \mathbf{M}}$ approaches the Identity operator.

    \item \textbf{Gradient of Core $\mathbf{M}$:} The element-wise derivative is:
    $$ \frac{\partial \mathbf{M}_{ij}}{\partial \mW_{ij}} = 2\alpha \mW_{ij} + (1-\alpha) \text{sign}(\mW_{ij}) $$
    As $\mW_{ij} \to 0$, the quadratic term $2\alpha \mW_{ij}$ vanishes. However, the sign term $\text{sign}(\mW_{ij})$ remains $\pm 1$ (for non-zero entries approaching zero).
\end{enumerate}

\textbf{Conclusion:} Combining these limits, for any direction $\mU$ (where $\mW = \delta \mU$), the gradient behaves as:
$$ \lim_{\mW \to \mathbf{0}} (\nabla_\mW h)_{ij} \approx 1 \cdot 1 \cdot (0 + (1-\alpha)\text{sign}(\mW_{ij})) = (1-\alpha)\text{sign}(\mW_{ij}) $$
This vector has a Frobenius norm of $(1-\alpha)\sqrt{\text{nnz}(\mW)}$, which is strictly positive provided $\alpha < 1$. This formally proves adherence to Axiom \ref{axiom:l1_synergy}.
\section{Proof of Finite-Time Structure Recovery (Theorem \ref{thm:finite_time_recovery})}
\label{app:proof_recovery}

\subsection{Notation and Strategy}
Let $\hat{\mW}$ be the stationary point (local minimizer) of the objective function $F(\mW) = \mathcal{L}(\mW) + \lambda_1 \|\mW\|_1$.
The proof proceeds in two critical stages:
\begin{enumerate}
    \item \textbf{Stage 1 (Statistical Consistency):} We prove that under the stated assumptions, the stationary point $\hat{\mW}$ is unique in the local neighborhood of $\mW^*$ and shares the exact support of the ground truth, i.e., $\text{sign}(\hat{\mW}) = \text{sign}(\mW^*)$. We use the \textit{Primal-Dual Witness Construction} technique \citep{wainwright2009sharp}.
    \item \textbf{Stage 2 (Finite-Time Identification):} We prove that the SPG algorithm generates a sequence $\{\mW_k\}$ that identifies the support of $\hat{\mW}$ in finite steps. We rely on the \textit{Manifold Identification Property} of proximal algorithms \citep{lewis2011active}.
\end{enumerate}

\subsection{Stage 1: Statistical Support Consistency}

We construct a "witness" solution $\tilde{\mW}$ restricted to the true support $S = \text{supp}(\mW^*)$ and show that it satisfies the optimality conditions for the full problem.

\textbf{Step 1.1: Restricted Optimization.}
Let $\tilde{\mW}$ be the solution to the restricted problem:
$$ \tilde{\mW} = \argmin_{\mW: \mW_{S^c} = 0} \mathcal{L}(\mW) + \lambda_1 \|\mW\|_1 $$
By construction, $\tilde{\mW}_{S^c} = 0$. For $(i,j) \in S$, the KKT condition implies:
$$ \nabla \mathcal{L}(\tilde{\mW})_S + \lambda_1 \text{sign}(\tilde{\mW})_S = 0 $$

\textbf{Step 1.2: Bound on Estimation Error.}
Using the Mean Value Theorem, $\nabla \mathcal{L}(\tilde{\mW})_S = \nabla \mathcal{L}(\mW^*)_S + \nabla^2 \mathcal{L}(\bar{\mW})_{SS} (\tilde{\mW} - \mW^*)_S$.
Substituting this into the KKT condition and utilizing the RSC (Assumption \ref{asm:rsc}) which bounds the minimum eigenvalue of the Hessian $\nabla^2 \mathcal{L}$, we can derive the standard error bound:
$$ \|\tilde{\mW} - \mW^*\|_\infty \le \|\tilde{\mW} - \mW^*\|_F \le \frac{2 \lambda_1 \sqrt{|S|}}{\kappa} $$
(Note: For detailed derivation of this bound from RSC, see \citet{Negahban2012}).

\textbf{Step 1.3: Verify Strict Dual Feasibility (Quantifying Sample Complexity).}
For $\tilde{\mW}$ to be the global solution, we must check the strict dual feasibility condition on the zero set $S^c$: $|\nabla \mathcal{L}(\tilde{\mW})_{ij}| < \lambda_1$ for all $(i,j) \in S^c$.
Using Taylor expansion and the Irrepresentable Condition, strict dual feasibility holds if the noise term is bounded: $\|\nabla \mathcal{L}_{\text{fit}}(\mW^*)_{S^c}\|_\infty < \frac{\gamma \lambda_1}{2}$.

Assuming the noise in the linear SEM is sub-Gaussian, standard high-dimensional statistical theory \citep[e.g.,][]{Wainwright2019} provides the concentration inequality:
$$ P\left(\|\nabla \mathcal{L}_{\text{fit}}(\mW^*)_{S^c}\|_\infty \ge t\right) \le c_1 \exp\left(-c_2 N t^2\right) $$
Setting the threshold $t = \frac{\gamma \lambda_1}{2}$, we require the failure probability to be negligible ($\le \delta_{prob}$). This implies a lower bound on the sample size $N$:
$$ c_1 \exp\left(-c_2 N (\frac{\gamma \lambda_1}{2})^2\right) \le \delta_{prob} \implies N \ge \frac{4 \log(c_1/\delta_{prob})}{c_2 (\gamma \lambda_1)^2} $$
Thus, provided $N = \Omega(\frac{\log d}{\lambda_1^2})$, strict dual feasibility holds with high probability.

\textbf{Step 1.4: Verify Support Inclusion.}
Using the Beta-Min Condition (Assumption \ref{asm:beta_min}) and the error bound from Step 1.2:
$$ |\hat{\mW}_{ij}| \ge |\mW^*_{ij}| - |\hat{\mW}_{ij} - \mW^*_{ij}| \ge \frac{4\lambda_1}{\kappa} - \frac{2\lambda_1}{\kappa} > 0, \quad \forall (i,j) \in S $$
Thus, $\text{supp}(\hat{\mW}) = S = \text{supp}(\mW^*)$.

\subsection{Stage 2: Finite-Time Algorithmic Identification}

Having established that the stationary point $\hat{\mW}$ has the correct support, we now show SPG-AHOC finds it in finite time.
Let the objective be $F(\mW) = \mathcal{L}(\mW) + \lambda_1 \|\mW\|_1$. The update rule is:
$$ \mW_{k+1} = \text{prox}_{\eta \lambda_1} (\mW_k - \eta \nabla \mathcal{L}(\mW_k)) $$

\textbf{Step 2.1: Non-Degeneracy Condition.}
From Step 1.3, we established the \textit{Strict Dual Feasibility} for the stationary point $\hat{\mW}$:
\begin{equation}
    \label{eq:strict_dual}
    |\nabla \mathcal{L}(\hat{\mW})_{ij}| \le \lambda_1 - \tau, \quad \forall (i,j) \in S^c
\end{equation}
where $\tau = \gamma \lambda_1 / 2 > 0$. This is crucial. It means the "force" pushing zero entries away from zero is strictly less than the regularization threshold.

\textbf{Step 2.2: Convergence of Gradients.}
Since $\mathcal{L}$ is smooth ($C^2$ continuous) and $\mW_k \to \hat{\mW}$ (asymptotic convergence of Proximal Gradient), the gradients converge: $\nabla \mathcal{L}(\mW_k) \to \nabla \mathcal{L}(\hat{\mW})$.
Therefore, there exists a finite iteration $K_1$ such that for all $k \ge K_1$:
$$ \|\nabla \mathcal{L}(\mW_k) - \nabla \mathcal{L}(\hat{\mW})\|_\infty < \frac{\tau}{2} $$

\textbf{Step 2.3: Zero Identification (Handling Smoothing Bias).}
Recall we optimize the smoothed constraint $\tilde{h}$ with parameter $\delta$. The gradient used in the update is $\nabla \tilde{\mathcal{L}} = \nabla \mathcal{L}_{\text{exact}} + \mathbf{E}_{bias}$, where $\norm{\mathbf{E}_{bias}}_\infty \le C \delta$ (Lipschitz property).
The proximal operator sets $(\mW_{k+1})_{ij} = 0$ if the input falls within $[-\lambda_1 \eta, \lambda_1 \eta]$.
From Step 2.1, the exact gradient on the zero set is bounded by $\lambda_1 - \tau$ (where $\tau = \gamma \lambda_1 / 2$). With smoothing bias, the effective gradient magnitude is at most:
$$ |\nabla \tilde{\mathcal{L}}_{ij}| \le (\lambda_1 - \tau) + C \delta $$
For the proximal operator to correctly identify zero, we need this magnitude to be strictly less than $\lambda_1$. This requires:
$$ C \delta < \tau \implies \delta < \frac{\gamma \lambda_1}{2C} $$
\textbf{Implication:} As long as the smoothing parameter $\delta$ is chosen to be $\mathcal{O}(\lambda_1)$, the bias is "swallowed" by the strict dual feasibility margin $\tau$, ensuring exact zero recovery despite the smooth approximation.

\textbf{Step 2.4: Non-Zero Identification.}
For $(i,j) \in S$, we know $|\hat{\mW}_{ij}| > 0$. By simple continuity, there exists $K_3$ such that for all $k \ge K_3$, $(\mW_k)_{ij}$ remains non-zero.

\subsection{Conclusion}
Let $K = \max(K_1, K_2, K_3)$. For all $k \ge K$, we have $\text{supp}(\mW_k) = \text{supp}(\hat{\mW}) = \text{supp}(\mW^*)$.
This completes the proof of Finite-Time Oracle Property. \qed

\section{Comparison with DAGMA}
\label{app:dagma_analysis}
The DAGMA method \citep{Bello2022DAGMA} employs a log-determinant constraint $h_{\text{DAGMA}}(\mW) = -\log \det (s\matI - \mW \hadamard \mW) + d \log s$. While DAGMA solves Type II SCI via the domain of the log-det function, it operates on a different theoretical principle by reformulating the constrained problem into an unconstrained one. SPG-AHOC provides a complementary approach: achieving stability and completeness through the constraint's definition (Axioms 1-3) and smooth approximation within the ALM framework.

\section{Optimizer Safety Nets}
\label{app:safety_nets}
Our experiments showed Proximal-LS appeared robust with $h_{\logdet}$. This is an artifact of the optimizer's line search "safety net". If a step lands where $h(\mW) = \infty$ (undefined), the line search rejects it. This is not intrinsic constraint robustness (Axiom 3) but an implementation detail. SPG-AHOC satisfies Axiom 3 intrinsically, ensuring $h(\mW)$ is always finite and differentiable (after smoothing), requiring no external safety checks.

\section{Convergence and Approximation Proofs (SPG-AHOC)}
\label{app:spg_proof_full}
\subsection{Boundedness of Optimization Trajectory}
\label{app:bounded_trajectory}

Before deriving the Lipschitz constants and approximation bounds, we first establish that the optimization iterates remain within a compact set.

\begin{lemma}[Bounded Trajectory]
    \label{lem:bounded_trajectory}
    Suppose the stability condition $\lambda_1 \ge \norm{\nabla \mathcal{L}_{\text{fit}}(\mathbf{0})}_\infty$ holds (Proposition \ref{prop:stability_condition}). Let $L_\delta(\mW) = \tilde{L}(\mW) + \lambda_1 \norm{\mW}_1$ be the objective function.
    The sublevel set $\mathcal{S}_0 = \{ \mW : L_\delta(\mW) \le L_\delta(\mathbf{0}) \}$ is compact. Consequently, the sequence $\{\mW_k\}$ generated by SPG-AHOC satisfies:
    \begin{equation}
        \sup_{k \ge 0} \norm{\mW_k}_F \le R < \infty
    \end{equation}
    for some finite constant $R$ depending on the initialization and $\lambda_1$.
\end{lemma}

\begin{proof}
    At initialization, $\mW_0 = \mathbf{0}$. The objective value is $L_\delta(\mathbf{0}) = \mathcal{L}_{\text{fit}}(\mathbf{0})$.
    Since the constraint term in $\tilde{L}$ is non-negative and the loss $\mathcal{L}_{\text{fit}}$ is typically coercive (e.g., quadratic for linear SEMs) or bounded below, the term $\lambda_1 \norm{\mW}_1$ dominates the growth of the objective as $\norm{\mW} \to \infty$.
    Specifically, the coercivity of the $L_1$ regularizer ensures that $\lim_{\norm{\mW} \to \infty} L_\delta(\mW) = \infty$.
    Since SPG is a descent method (Theorem \ref{thm:spg_convergence}), $\mW_k \in \mathcal{S}_0$ for all $k$. The compactness of $\mathcal{S}_0$ implies uniform boundedness of the iterates.
    This boundedness ensures that the Lipschitz constant $L_\delta \approx \sup_{\mW \in \mathcal{S}_0} \norm{\nabla^2 \tilde{h}(\mW)}$ is finite.
\end{proof}
\subsection{Proof of Theorem \ref{thm:spg_convergence} (Lipschitz Gradient)}
\label{app:proof_spg_convergence}

\begin{proof}
    The objective uses the smoothed absolute value $s(x) = \sqrt{x^2 + \delta^2}$.
    The Lipschitz constant of the gradient ($L_\delta$) is determined by the maximum eigenvalue of the Hessian $\nabla^2 \tilde{L}$.
    Applying the chain rule to the composite function $h(\mW) = g(\phi(s(\mW)))$, the Hessian contains terms involving the first and second derivatives of the components.

    The dominant term comes from the second derivative of the smoothing function $s(x)$:
    $$ s''(x) = \frac{\delta^2}{(x^2 + \delta^2)^{3/2}} $$
    This function achieves its maximum at $x=0$, where $s''(0) = \frac{\delta^2}{\delta^3} = \frac{1}{\delta}$.

    Other components (matrix exponential $g$, normalization $\phi$) are smooth and have bounded derivatives on the compact unit ball (enforced by $\phi$). Let $C_{rest}$ be the bound for the product of these other derivatives. The Lipschitz constant is bounded by:
    $$ L_\delta \le C_{rest} \cdot \sup_x |s''(x)| = \frac{C_{rest}}{\delta} $$

    Therefore, $L_\delta = \mathcal{O}(1/\delta)$. To ensure convergence of the Proximal Gradient method (Descent Lemma), the step size must satisfy $\eta \le \frac{1}{L_\delta}$, which implies $\eta \propto \delta$.
\end{proof}

\subsection{Proof of Proposition \ref{prop:approx_bound} (Approximation Bound)}
\begin{proof}
    We bound $\abs{\tilde{h}(\mW) - h(\mW)}$.
    The smoothing error is in the core: $\abs{\smoothabs{x}{\delta} - |x|} \le \delta$.
    Thus $\norm{\tilde{\mathbf{M}} - \mathbf{M}}_F \le (1-\alpha) d \delta$.
    The normalization map and matrix exponential are Lipschitz continuous on the relevant domains. Let their combined Lipschitz constant be $L$.
    Then $\abs{\tilde{h}(\mW) - h(\mW)} \le L \norm{\tilde{\mathbf{M}} - \mathbf{M}}_F \le L (1-\alpha) d \delta$.
    This confirms the uniform approximation bound $\mO(\delta)$.
\end{proof}

\section{Formal Derivation of Stability Condition}
\label{app:formal_lambda1_proof}

\begin{proof}
    A point $\mW^*$ is a stationary point of the objective $L_\delta(\mW) = \tilde{L}(\mW) + \lambda_1 \norm{\mW}_1$ if and only if:
    $$ \mathbf{0} \in \nabla \tilde{L}(\mW^*) + \lambda_1 \partial \norm{\mW^*}_1 $$
    where $\partial \norm{\cdot}_1$ is the subdifferential of the element-wise $L_1$ norm.
    At initialization $\mW_0 = \mathbf{0}$:
    1.  Vanishing Constraint Gradient: Due to the smoothing $s(x) = \sqrt{x^2+\delta^2}$, we have $s'(0)=0$. Consequently, $\nabla \tilde{h}_{\text{AHOC}}(\mathbf{0}) = \mathbf{0}$. The gradient of the smooth part reduces to the data fidelity term: $\nabla \tilde{L}(\mathbf{0}) = \nabla \mathcal{L}_{\text{fit}}(\mathbf{0})$.
    2.  Subdifferential at Origin: The subdifferential of the $L_1$ norm at the origin is the unit hypercube:
        $$ \partial \norm{\mathbf{0}}_1 = \{ \mathbf{G} \in \R^{d \times d} : \abs{\mathbf{G}_{ij}} \le 1, \forall i,j \} $$

    Substituting these into the stationarity condition, we require:
    $$ -\nabla \mathcal{L}_{\text{fit}}(\mathbf{0}) \in \lambda_1 \partial \norm{\mathbf{0}}_1 $$
    $$ \iff \abs{[\nabla \mathcal{L}_{\text{fit}}(\mathbf{0})]_{ij}} \le \lambda_1, \quad \forall i,j $$
    $$ \iff \norm{\nabla \mathcal{L}_{\text{fit}}(\mathbf{0})}_\infty \le \lambda_1 $$
    If this condition is violated, the negative gradient direction lies outside the subdifferential cone, forcing the optimizer to take a non-zero step away from $\mathbf{0}$, thus breaking sparsity immediately.
\end{proof}

\section{Additional Visualizations}
\label{app:lambda1_viz}

To visualize the stabilizer effect discussed in Section \ref{subsec:results_exp7}, Figure \ref{fig:exp7_lambda1_trajectory} tracks the optimization trajectory during the first outer loop.

\begin{figure*}[ht]
    \vskip 0.2in
    \begin{center}
        \includegraphics[width=1.0\textwidth]{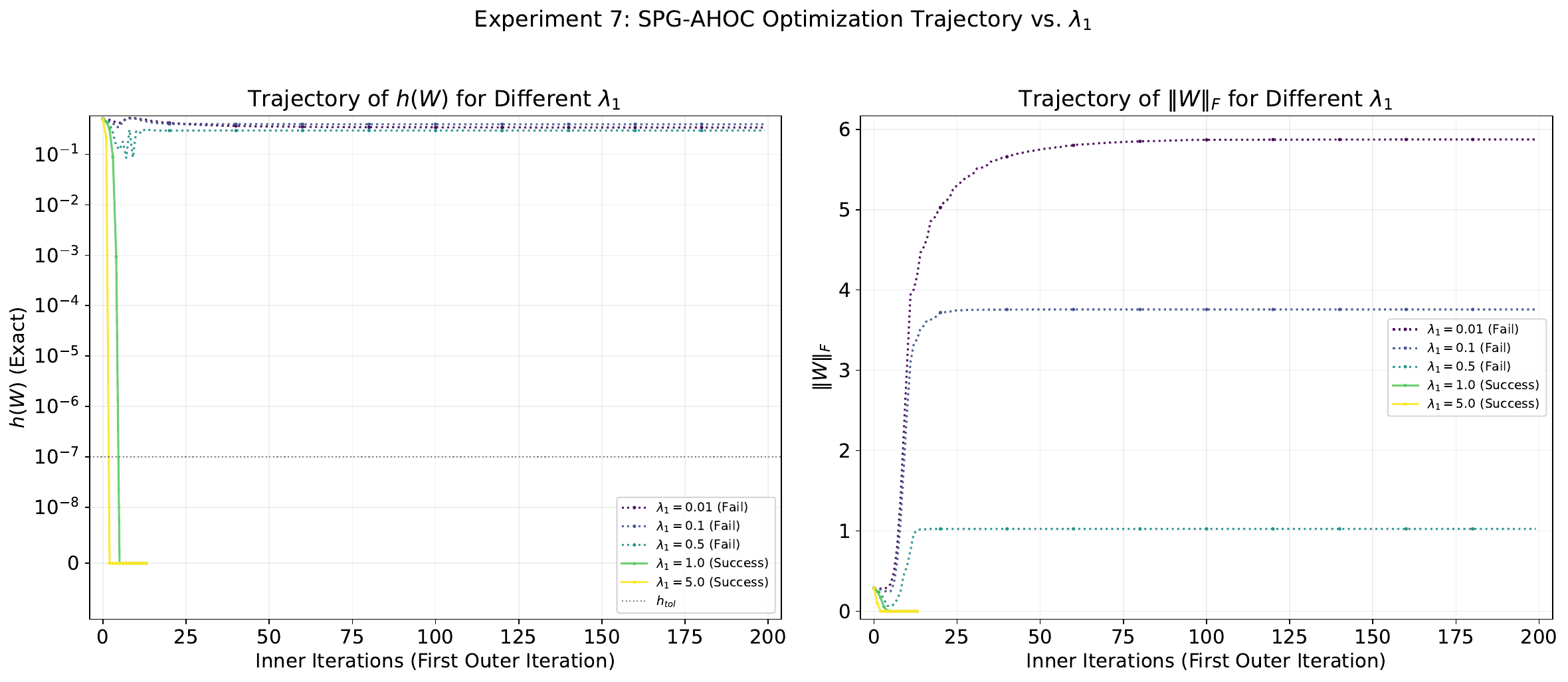}
        \caption{\textbf{Visualizing the Stabilizer Effect (Exp 7).} Optimization trajectory (first outer loop) vs. $\lambda_1$. Small $\lambda_1$ values (e.g., 0.1) lead to a rapid increase in $\norm{\mW}_F$, causing the optimizer to drift into non-convex regions and eventually fail. Larger $\lambda_1$ values (e.g., 1.0, 2.0) effectively constrain $\norm{\mW}_F$ near zero initially, satisfying the stability condition and enabling convergence.}
        \label{fig:exp7_lambda1_trajectory}
    \end{center}
    \vskip -0.2in
\end{figure*}

As confirmed by the trajectories in Figure \ref{fig:exp7_lambda1_trajectory}, when $\lambda_1 < \norm{\nabla \mathcal{L}_{\text{fit}}(\mathbf{0})}_\infty$ (violating the stability condition defined in Proposition \ref{prop:stability_condition}), the iterates immediately break away from the sparse manifold at the origin. This confirms that sufficient regularization is a prerequisite for numerical stability in this continuous framework.
\section{Stability and Limits in High Dimensions (\texorpdfstring{$d=500$}{d=500})}
\label{app:high_dim_stability}

We rigorously tested the theoretical limits of SPG-AHOC on high-dimensional graphs ($d=500, e=500$) with limited samples ($n=1000$). This setting poses significant challenges due to the $O(d^3)$ complexity of the matrix exponential and the non-convex landscape.

\textbf{Empirical Findings:}
\begin{itemize}
    \item \textbf{Baseline Failure:} The standard Adam-based augmented Lagrangian method (EXP) failed to complete the task, exhibiting prohibitive runtime/memory costs (specifically, Out-Of-Memory errors).
    \item \textbf{SPG-AHOC Viability:} Our method successfully converged. We observed a distinct \textbf{Stability-Sensitivity Trade-off} governed by the geometric stabilizer $\lambda_1$:
    \begin{enumerate}
        \item \textit{Weak Stabilization ($\lambda_1 \le 0.1$):} The algorithm diverged (LS Fail/NaN), confirming that the stability condition $\|\nabla \mathcal{L}_{\text{fit}}(\mathbf{0})\|_\infty \le \lambda_1$ is necessary.
        \item \textit{Optimal Regime ($\lambda_1 = 0.5$):} The algorithm achieved an optimal stability-sensitivity balance, converging to a valid DAG with \textbf{TPR=0.44}. This indicates that SPG-AHOC can effectively recover causal structures in high dimensions where standard methods fail.
        \item \textit{Over-Stabilization ($\lambda_1 \ge 1.0$):} The algorithm converged rapidly (< 4s) but yielded an empty graph (TPR=0.0), as the stabilizer strength suppressed all edge signals.
    \end{enumerate}
\end{itemize}

\textbf{Insight:} These results suggest that for massive-scale causal discovery, $\lambda_1$ plays a dual role: it must be large enough to ensure numerical stability (anchoring the trajectory) but small enough to permit signal detection. SPG-AHOC provides a viable path forward in this challenging regime.
\section{Theoretical Analysis of Near-Cyclic Instability}
\label{app:near_cyclic_theory}

We provide a theoretical justification for the failure of AHOC on near-cyclic graphs (Experiment 4, $\rho \approx 1$), despite satisfying the Numerical Boundedness axiom.

While the normalization $\phi(\mW)$ ensures the \textit{value} of $h(\mW)$ and its \textit{gradient} $\nabla h$ are bounded, the \textbf{curvature} (Hessian) becomes ill-conditioned.
Consider the spectral radius formulation. As $\mW$ approaches a cycle, the spectral radius $\rho(\mW \hadamard \mW) \to 1$.
The derivative of the matrix exponential involves the term $\expm(\mathbf{M})$. Even with normalization, the geometry of the feasible set $h(\mW)=0$ forms a "thin manifold" near cyclic configurations.

Specifically, for the AHOC constraint, let $\mathbf{M}$ be the core matrix. Near the boundary of the acyclic polytope:
\begin{enumerate}
    \item The condition number of the Hessian $\kappa(\nabla^2 h)$ diverges.
    \item The "valley" of the optimization landscape becomes extremely narrow.
    \item Standard first-order methods (like Proximal Gradient) bounce between the walls of this narrow valley unless the step size is infinitesimally small ($\eta \to 0$).
\end{enumerate}

Unlike methods with a logarithmic barrier (like $h_{\log\det} = -\log\det(I-\mA)$) which provides a smooth repulsive force approaching infinity as $\rho \to 1$ (effectively a "safety net"), the AHOC constraint is purely penalty-based. Without the barrier property, the optimizer can overshoot into the cyclic region where the gradient direction may drastically shift, leading to the oscillations observed in Figure \ref{fig:exp4_large}.

\section{Detailed Axiomatic Validation Plots}
\label{app:axiomatic_plots}

In this section, we provide enlarged views of the gradient behaviors to validate the Structural Constraint Instability (SCI) analysis and the $L_1$-Synergy axiom.

\begin{figure}[h]
    \centering
    \includegraphics[width=0.85\textwidth]{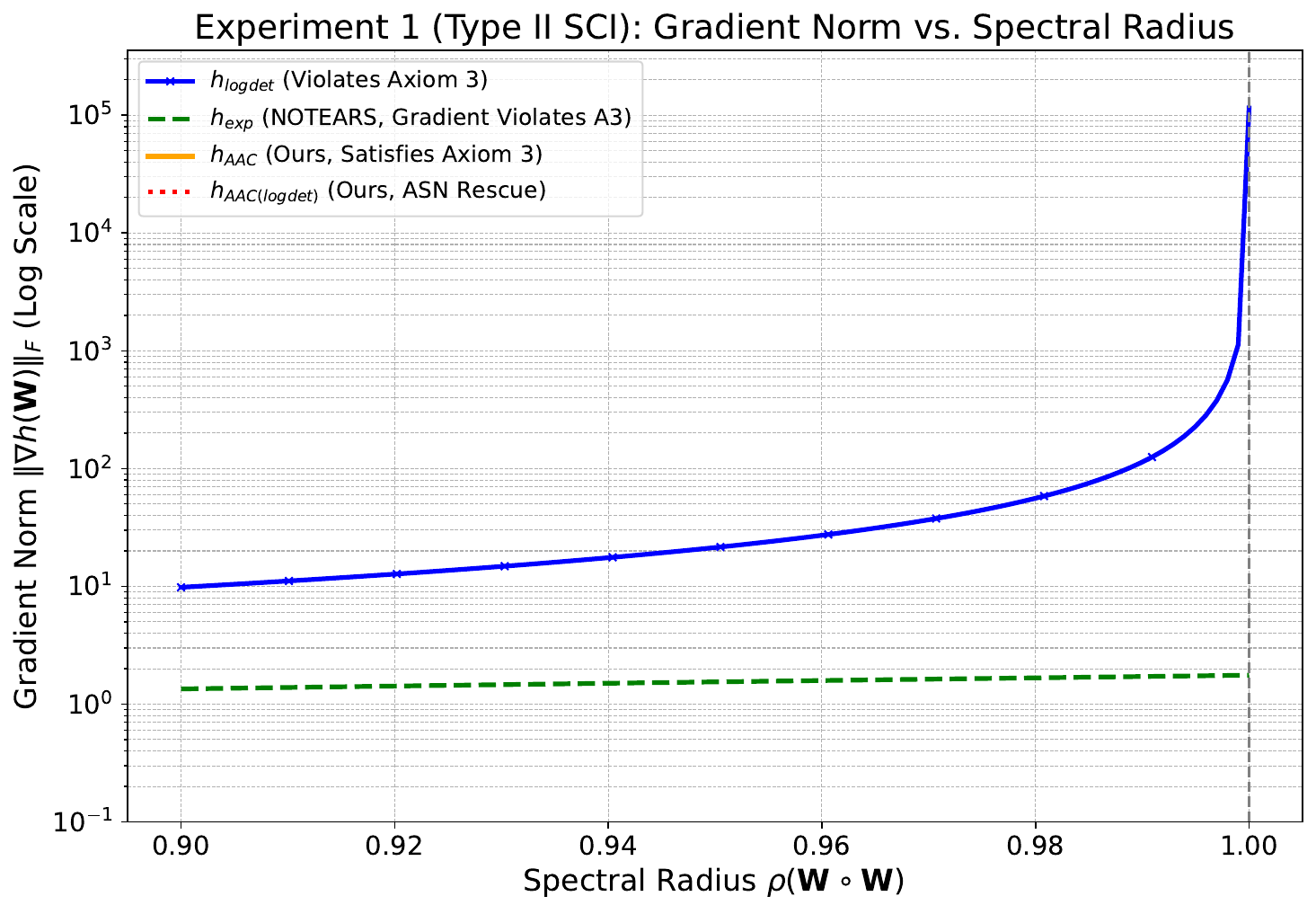}
    \caption{\textbf{Experiment 1 (Type II SCI): Gradient Norm vs. Spectral Radius.}
    Comparison of gradient behaviors as the matrix approaches a cycle ($\rho \to 1$). Standard constraints explode, whereas AHOC and AAC remain bounded (Axiom 3).}
    \label{fig:exp1_large}
\end{figure}

\begin{figure}[h]
    \centering
    \includegraphics[width=0.6\textwidth]{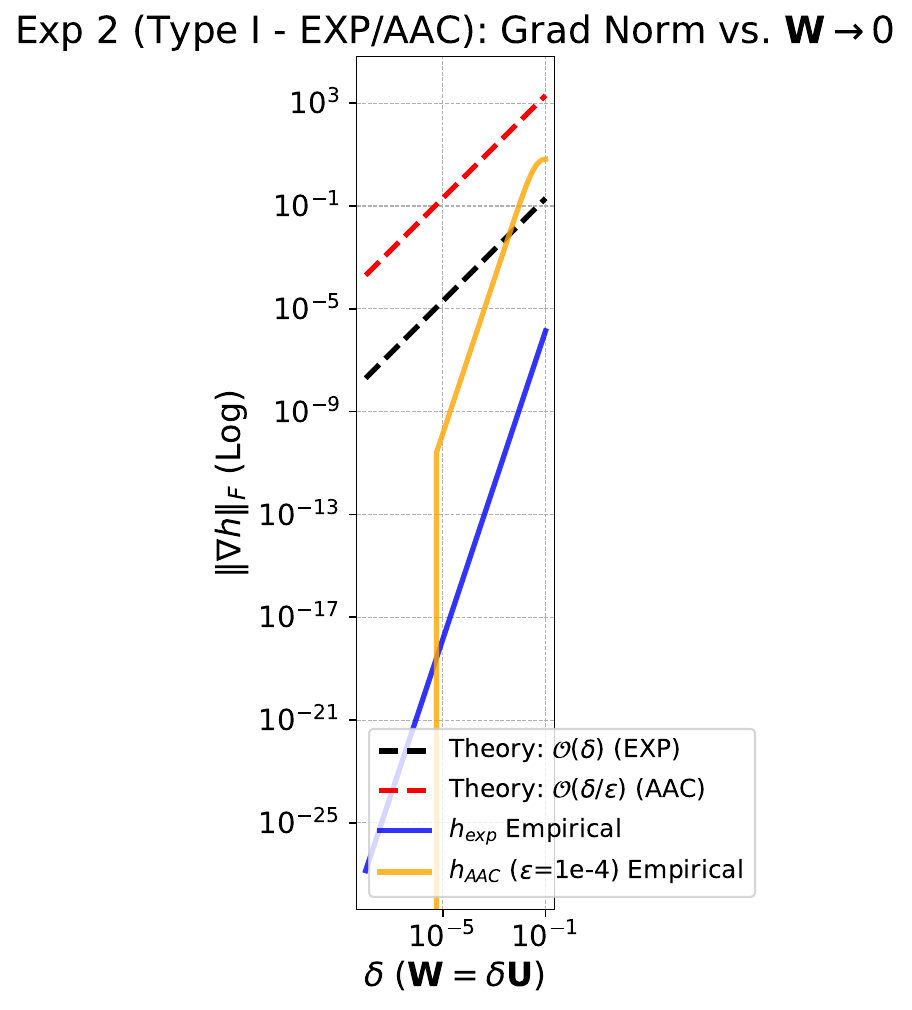}
    \caption{\textbf{Experiment 2a (Type I SCI): Gradient Norm vs. Weight Magnitude.}
    As $\mW \to \mathbf{0}$, the gradients of standard constraints (EXP, AAC) vanish (dashed lines), violating Axiom 2. This leads to optimization stagnation in sparse regimes.}
    \label{fig:exp2a_large}
\end{figure}

\begin{figure}[h]
    \centering
    \includegraphics[width=0.85\textwidth]{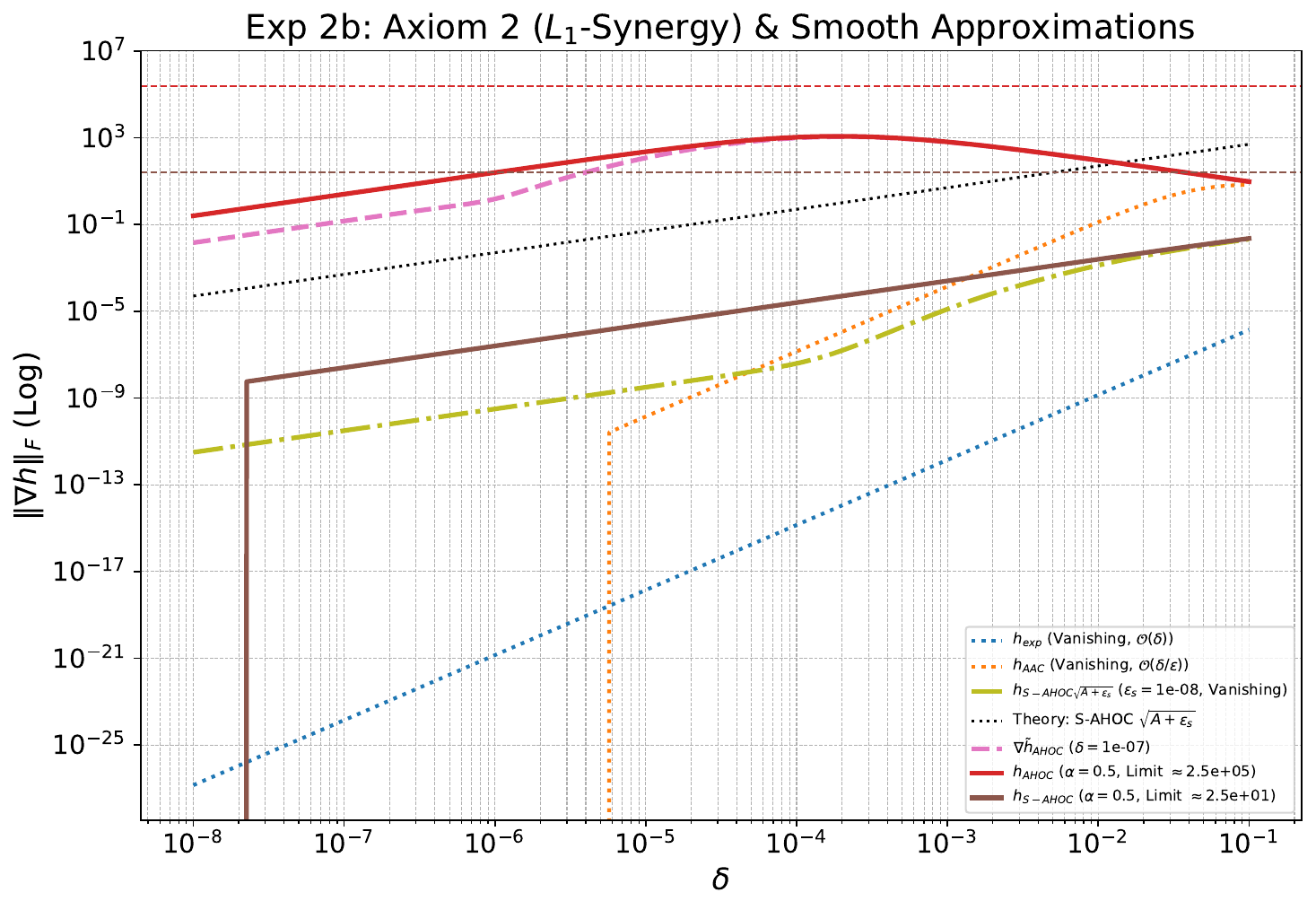}
    \caption{\textbf{Experiment 2b (Axiom 2: $L_1$-Synergy).}
    Validation of the AHOC family. Unlike standard methods, AHOC and S-AHOC maintain a strictly positive gradient norm lower bound even as $\mW \to \mathbf{0}$ (solid lines), ensuring compatibility with $L_1$ regularization.}
    \label{fig:exp2b_large}
\end{figure}

\section{Detailed Robustness Stress Test Plot}
\label{app:robustness_plots}

In this section, we provide the enlarged convergence history for Experiment 4 to clearly distinguish the trajectories of different methods under extreme near-cyclic conditions.

\begin{figure}[h]
    \centering
    \includegraphics[width=1.0\textwidth]{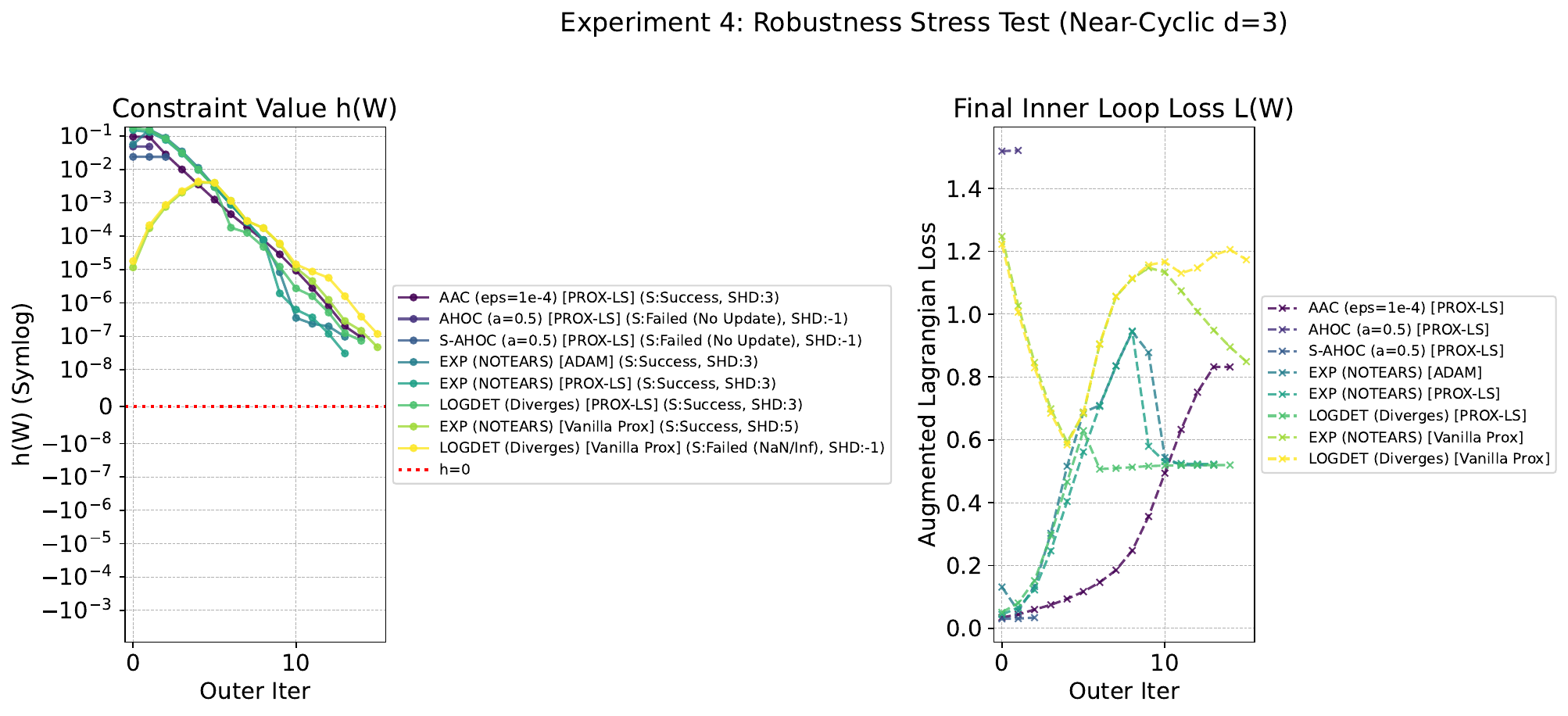}
    \caption{\textbf{Experiment 4: ALM Optimization History on Near-Cyclic Graph.}
    \textit{Left Panel:} The constraint violation $h(\mW)$. Note that the AHOC family (blue/purple lines stuck near bottom) fails to make progress, while AAC and NOTEARS (green/yellow) successfully reduce the violation.
    \textit{Right Panel:} The Augmented Lagrangian objective. The divergence of AHOC highlights the numerical challenges in regions where $\rho(\mW \circ \mW) \to 1$.}
    \label{fig:exp4_large}
\end{figure}

\end{document}